\documentclass[times, review, 10pt]{elsarticle}

\journal{Pattern Recognition}
\usepackage{algorithmic}
\usepackage{algorithm}
\usepackage{multicol,multirow}
\usepackage{amsmath, amsthm, amscd, amsfonts, stmaryrd, amssymb}
\usepackage{mathrsfs}
\usepackage{amsthm}
\usepackage{rotating}
\usepackage{ifpdf}
\usepackage[T1]{fontenc}
\usepackage{newtxtext}
\usepackage{newtxmath}
\usepackage{textcomp}
\usepackage{tikz}
\usepackage{booktabs}
\usepackage{makecell}
\usepackage{threeparttable}
\usepackage[english]{babel}
\usepackage{graphicx}
\usepackage{subfig}
\usepackage{geometry}

\usepackage{leftidx}%
\usepackage[displaymath, mathlines]{lineno}
\usepackage{accents}%
\usepackage{enumerate}   
\usepackage{color}
\usepackage{hyperref}
\usepackage[german=guillemets]{csquotes}
\usetikzlibrary{matrix}
\usepackage{caption}
\usepackage{amscd}
\usepackage{tabularx}
\usepackage{float}
\usepackage{booktabs}%
\usepackage{newtxmath}
\usepackage{dsfont}
\usepackage{hyperref}
\usepackage{indentfirst}
\allowdisplaybreaks[4]

\newcommand{\mR}{\mathbb{R}}

\newtheorem{theorem}{Theorem}[section]
\newtheorem{lemma}[theorem]{Lemma}

\theoremstyle{definition}
\newtheorem{definition}[theorem]{Definition}

\newtheorem{pro}[theorem]{Problem}

\newtheorem{proposition}[theorem]{Proposition}
\theoremstyle{remark}

\numberwithin{equation}{section}

\newcommand{\ABS}[1]{\left| {#1} \right|}
\newcommand{\NORM}[1]{\left\| {#1} \right\|}

\newcommand{\bA}{\boldsymbol{A}}
\newcommand{\bB}{\boldsymbol{B}}
\newcommand{\bz}{\boldsymbol{z}}
\newcommand{\bx}{\boldsymbol{x}}
\newcommand{\bw}{\boldsymbol{\omega}}
\newcommand{\bq}{\boldsymbol{q}}
\newcommand{\bp}{\boldsymbol{p}}
\newcommand{\bal}{\boldsymbol{\alpha}}
\newcommand{\bSi}{\boldsymbol{\Sigma}}
\newcommand{\bomega}{\boldsymbol{\omega}}
\newcommand{\bv}{\boldsymbol{v}}
\allowdisplaybreaks[4]

\begin{document}

\begin{frontmatter}



\title{{Quaternionic Reweighted Amplitude Flow for Phase Retrieval in Image Reconstruction}} 


\author{Ren Hu} 

\affiliation{organization={Department of Electronics and Information Systems, Ghent University},
            addressline={Krijgslaan 281, Building S8}, 
            city={Gent},
            postcode={9000}, 
            state={Gent},
            country={Belgium}}

\author{Pan Lian} 

\affiliation{organization={School of Mathematical Sciences, Tianjin Normal University},
	addressline={Binshui No.393, Xiqing District}, 
	city={Tianjin},
	postcode={300387}, 
	state={Tianjin},
	country={China}}

\begin{abstract}
Quaternionic signal processing provides powerful tools for efficiently managing color signals by preserving the intrinsic correlations among signal dimensions through quaternion algebra. In this paper, we address the quaternionic phase retrieval problem by systematically developing novel algorithms based on an amplitude-based model. Specifically, we propose the Quaternionic Reweighted Amplitude Flow (QRAF) algorithm, which is further enhanced by three of its variants: incremental, accelerated, and adapted QRAF algorithms. In addition, we introduce the Quaternionic Perturbed Amplitude Flow (QPAF) algorithm, which has linear convergence. Extensive numerical experiments on  both synthetic data and real images, demonstrate that our proposed methods significantly improve recovery performance and computational efficiency compared to state-of-the-art approaches.
\end{abstract}


\end{frontmatter}


\section{{Introduction}}

Hypercomplex algebras, such as quaternions, octonions, and Clifford algebras, have demonstrated considerable utility in the representation and processing of multidimensional signals. These algebraic systems extend the capabilities of conventional vector spaces by enabling not only scalar multiplication and addition, but also non-commutative multiplication, allowing for compact and structured representations of color, geometric, and directional information. In particular, quaternion algebra has become increasingly important in color image processing \cite{ldp}, signal fusion \cite{MFLC} and neural networks \cite{pml}, as it naturally encodes RGB channels in a single algebraic object, thereby preserving inter-channel correlations \cite{MFLC}.

In recent years, quaternion-based methods have been increasingly applied to ill-posed inverse problems in imaging. Notably, a range of quaternion models have been introduced for color image denoising and inpainting. Quaternion Matrix Completion (QMC) \cite{QMC} exploits the low-rank structure of quaternion-encoded images for recovering missing pixels; Nonlocal Self-Similarity QMC (NSS-QMC) \cite{NSSQMC} extends this by leveraging patch-based nonlocal priors to recover fine texture details; Saturation-Value Total Variation (SVTV) \cite{SVTV} applies a perceptual total variation regularization in the HSV color space, preserving color edges more faithfully; and Cross-Space Total Variation with Quaternion Blur Operator (CSTV-QBO) \cite{CSTV} introduces a novel model that simultaneously enforces TV constraints in multiple color spaces and models inter-channel blur using quaternion convolution. These methods demonstrate that quaternionic modeling not only enhances color consistency but also enables structure-preserving priors that are difficult to formulate in real-valued domains.

Inspired by this progress, researchers have recently begun extending classical phase retrieval (PR) to hypercomplex settings, see e.g., \cite{jms}. In the conventional PR problem, the objective is to recover a signal $\boldsymbol{x}$ from phaseless measurements of the form $b_j = |\langle \bal_j, \boldsymbol{x} \rangle|$, see e.g. \cite{seccms}. This inverse problem is fundamentally nonconvex and arises in various imaging applications such as optics, crystallography, and diffraction imaging \cite{lt}. Numerous algorithms have been proposed in the real and complex domains, including convex semi-definite relaxations \cite{csv}, optimization-based method such as weighted nuclear norm minimization (WNNM) method \cite{LI2022108537}, gradient-based methods such as Wirtinger Flow (WF) \cite{candes_phase_2015_Theory}, and amplitude-based refinements such as Truncated Amplitude Flow (TAF) and Reweighted Amplitude Flow (RAF) \cite{wang_phase_2018}.

However, extending these methods to the quaternionic domain introduces new challenges that cannot be trivially addressed. Unlike the complex field, the quaternion algebra is non-commutative, which complicates the gradient-based analysis. Moreover, many foundational tools—such as Wirtinger calculus and spectral theorem — require significant adaptation. 
Significant initial progress on quaternion phase retrieval (QPR) has been made by Chen and Ng in \cite{chen_phase_2023}. They proposed a Quaternion Wirtinger Flow (QWF) algorithm using a generalized $\mathbb{HR}$ calculus developed in \cite{xup}, while the subsequent studies proposed Quaternion Truncated Wirtinger Flow (QTWF) and Quaternion Truncated Amplitude Flow (QTAF), which showed empirical improvements. More complicated Octonic phase retrieval was considered in \cite{jmsa}.
Nevertheless, these methods primarily focused on intensity-based models. 

\subsection{{Our Contributions}}

In this work, we consider the quaternion phase retrieval problem and systemically investigate the quaternion non-convex phase retrieval algorithms based on the quaternionic  amplitude-based model. Note that one key advantage of quaternionic methods is their ability to recover signals with substantially fewer measurements compared to real-valued methods based on monochromatic or concatenation models. Our algorithms are not exception. One of the primary difficulties in quaternionic phase retrieval stems from the non-commutative nature of quaternion multiplication, which complicates both algorithm design and theoretical analysis. Our proposed algorithms are designed to directly confront these issues. In particular, we leverage the quaternionic algebraic structure—rather than  analyze it via real or complex reformulations—to develop efficient and theoretically motivated methods, thereby preserving inter-channel correlations. Our contributions are two fold:

(i) We introduce the Reweighted Amplitude Flow algorithm for quaternion-valued signals (QRAF), which extends the Real/Complex Amplitude Flow (RAF) algorithm originally proposed in \cite{wang_solving_2018}. The non-commutativity of quaternion multiplication plays a central role in the algorithm's design, necessitating careful ordering of operations to preserve correctness and convergence.   Numerical experiments presented in Section \ref{p7} show that QRAF algorithm consistently outperforms existing methods, such as QWF, QTWF and QTAF \cite{chen_phase_2023}.  Notably, both the Quaternionic Reshaped Wirtinger Flow (QRWF) and Quaternionic Truncated Amplitude Flow (QTAF) can be seen  as special cases of QRAF.  Furthermore, three variants of QRAF based on the gradient decent are introduced, which significantly enhance its performance. A detailed discussion for the convergence of QRAF is given, concluding with an open question which will be solved  the corresponding complex-valued case is fully understood.

(ii) We also propose the Quaternionic Perturbed Amplitude Flow (QPAF). The QPAF algorithm needn't any truncation or re-weighted procedure, yet it achieves  comparable numerical performance. {Importantly, unlike the algorithms based on amplitude-based model}, the theoretical analysis of PAF is straightforward to extend to the present quaternionic setting. making it a promising alternative in practice.

{Since any quaternionic matrix can be represented using real or complex matrices, a natural and frequently asked question is that if our algorithms are merely rebranded versions of existing real/complex methods. In the context of quaternionic phase retrieval, this is not the case.  Indeed, it is not possible to recover a signal $\boldsymbol{x}$ from phaseless quaternionic measurements of the form $b_j = |\langle \bal_j, \boldsymbol{x} \rangle|$ using real or complex analogues alone. Our algorithms are intrinsically governed by the quaternionic algebraic structure. Even when implemented via real or complex matrix representations, the underlying quaternionic behavior remains critical to their performance and theoretical soundness.}

{\em Paper Organization and Notations} \,The rest of this paper is organized as follows. Section \ref{p2}  introduces the necessary  preliminaries.  In Section \ref{p3}, we present the QRAF algorithm, and provide a discussion for its  convergence. Section \ref{p4} explores three variants that further refine the QRAF algorithm. Section \ref{p5} focuses on the Quaternionic Perturbed Amplitude-based model and the non-convex QPAF algorithm. In Section \ref{p6}, we review a useful technique—phase factor estimation—for color image processing. Experimental results on both synthetic and real-world color image data are discussed in Section \ref{p7}.  Concluding remarks are given in Section \ref{p8} and future works are discussed in \ref{p10}. Throughout the paper, boldface lowercase letters such as $\bal_{i}, \boldsymbol{x}, \boldsymbol{z}$ denote vectors, and
boldface capital letters such as $\boldsymbol{A}, \boldsymbol{Y}$ denote matrices. For a quaternionic matrix and a vector in $\mathbb{H}^{d}$, while $\bA^*$ and $\bz^*$ denote conjugate
transposes of $\bA$ and $\bz$, respectively. For a matrix and a vector, $\bA^T$ and $\bz^T$ denote
transposes of $\bA$ and $\bz$, respectively. 
\section{Preliminaries} \label{p2}
\subsection{Quaternionic Matrices}
In mathematics, Hamilton's quaternion algebra, denoted by $\mathbb{H}$, extends the  familiar real and complex number fields.  A quaternion $q\in \mathbb{H}$  is typically represented as  
$
	q=q_a+ q_b i+ q_cj+ q_d k,
$
where $q_a, q_b, q_c, q_d\in \mathbb{R}$ and $i, j, k$ are  generalized imaginary units satisfy the relations $i^{2}=j^{2}=k^{2}=-1$ and $ij=-ji=k, jk=-kj=i, ki=-ik=j$. In this expression, the term $ q_bi+ q_cj+ q_d k$ is called the vector part of $q$, while $q_a$ is referred to as the  real  or scalar part. The conjugate of $q$ is defined as $\bar{q}=q_a- q_bi-q_cj-q_dk$. For any two quaternions $p$ and $q$, then $\overline{pq}=\bar{q} \bar{p}$, which is in general  not equal to  $\bar{p} \bar{q}$ again due to the non-commutativity of quaternion multiplication. The Euclidean norm of $q$ is given by
$
	|q|=\sqrt{q\bar{q}}=\sqrt{q_{a}^{2}+q_{b}^{2}+q_{c}^{2}+q_{d}^{2}}.
$

Let $\mathbb{H}^{d}$ denote the sets of $d$-dimensional quaternion vectors,  and  $\mathbb{H}^{d_{1}\times d_{2}}$  the sets of  $d_{1}\times d_{2} $  matrices with quaternionic entries. For  ${\bq}=[q_{k}]\in \mathbb{H}^{d}$ and the matrix ${\bA}=[q_{ij}]\in \mathbb{H}^{d_{1}\times d_{2}}$, the $\ell_{2}$ norm of ${\bq}$  is defined as  
$\|{\bq}\|=(\sum_{k=1}^{d}|q_{k}|^{2})^{1/2}$ and the matrix operator norm  of $\bA$ is given by 
$\|\bA\|=\sup_{\bomega\in \mathbb{H}^{d_{2}}\backslash \{0\} }\|{\bA}\bomega\|/\|\bomega\| $. Let ${\bf I}_{d}$
be the identity matrix. Similar to  real and complex matrices, a matrix ${\bA} \in \mathbb{H}^{d\times d}$ is called invertible if there  exists a matrix  ${\bB}$ such that ${\bA\bB}={\bB\bA}={\bf I}_{d}$. A matrix ${\bA} \in \mathbb{H}^{d\times d} $ is called Hermitian if ${\bA}^*={\bA}$,  and  unitary if ${\bA\bA^{*}}={\bA^{*}\bA={\bf I}_{d}}$, where ${\bA}^*$ denotes the conjugation transpose of ${\bA}$.

The eigenvalue and eigenvector theory of a quaternion matrix is more complicated than those for real or complex matrices due to the non-commutativity. As a result,  one can  consider the left and right eigenvalue equations separately. In this work, we focus on the right eigenvalue and eigenvector which has its physical significance.
Given ${\bA}\in \mathbb{H}^{d\times d}$, if ${\bA}{\bf x}={\bx}\lambda $ for some nonzero ${\bx} \in \mathbb{H}^{d}$, we refer $\lambda$, ${\bx}$ as the right eigenvalue and eigenvector of ${\bA}$. Note that ${\bA \bx}={\bx}\lambda$ is equal to ${\bA}( {\bx} v^{*})=(\bx)v^{*}(v\lambda v^{*})$ for any $v$  with modulus 1. Therefore, a matrix ${\bA}$ with eigenvalue $\lambda$ has a set of eigenvalues $\{v\lambda v^{*}\}$, from which we can select
a unique `standard eigenvalue' in the form of $a+bi$ with  $a\in\mathbb{R}$ and $ b\ge 0)$.  Any ${\bA}\in \mathbb{H}^{d\times d}$ has exactly $d$ standard eigenvalues, and in particular, all standard eigenvalues of Hermitian ${\bA}$
are real. Similar to complex Hermitian matrices, a quaternion Hermitian matrix ${\bA}$ can be decomposed as ${\bA}=\boldsymbol{U} \bSi \boldsymbol{U}^{*}$, where ${\boldsymbol{U} }$ is a  unitary matrix  and  ${\bSi}$ is diagonal, with the standard eigenvalues of ${\bA}$ arranged  in the diagonal of ${\bSi}$. The detailed proof of the results mentioned above of  quaternionic matrices  can be found in \cite{zhang}.
\subsection{Dirac Operator and Generalized $\mathbb{HR}$ Calculus}

In this subsection, we review the fundamentals of quaternion matrix derivatives. This framework is essential for developing gradient-descent-like iterations. It provides a comprehensive set of rules that enable the computation of derivatives for functions directly within the quaternion domain,  serving as an analogue to the Wirtinger calculus. 
For any quaternions $q$, consider the transformation
$
	q^{\mu}:=\mu q\mu^{-1}
$
where $\mu$ is any non-zero quaternion, which represents a $3$-dimensional rotation of the vector part of $q$.  The set $\{1, i^{\mu}, j^{\mu}, k^{\mu}\}$ forms a generalized orthogonal basis for $\mathbb{H}$. Similarly, for any quaternion vector $\bq\in \mathbb{H}^{n}$, we define its transformation as $\bq^{\mu}=\left(q_{1}^{\mu}, q_{2}^{\mu}, \cdots, q_{n}^{\mu}\right)$. The generalized $\mathbb{HR}$ calculus can be derived either from  left or right GHR derivatives. In this work, we will only use the left one.  Note that in mathematical literature (see e.g., \cite{dss}), the  derivative used here is often referred  as the right derivative.
\begin{definition} \cite{xup} The generalized $\mathbb{HR}$ derivative of a function $f$ with respect to the transformed quaternion $q^{\mu}$ is defined as 
	{\begin{equation*}
		\frac{\partial f}{\partial q^{\mu}}
		=\frac{1}{4}\left(\frac{\partial f}{\partial q_a }-\frac{\partial f}{\partial q_b}i^{\mu} - \frac{\partial f}{\partial q_c} j^{\mu}-\frac{\partial f}{\partial q_d}  k^{\mu}\right),
	\end{equation*} }
	where $\partial f/\partial q_{a}$, $\partial f/\partial q_{b}$, $\partial f/\partial q_{c}$, and $\partial f/\partial q_{d}$ are the partial derivatives of $f$ with respect to $q_{a}$, $q_{b}$, $q_{c}$, and $q_{d}$ respectively.
\end{definition}

\begin{definition}
	For a scalar function $f({\bq})$ with ${\bq}\in \mathbb{H}^{n}$,
	the gradient of $f$ with respect to ${\bq}$ is defined as 
	\begin{equation} \label{gedr}
		\nabla_{{\bq}^{\mu}}f=\left( \frac{\partial f}{\partial \bq^{\mu}}\right)^{*} \in \mathbb{H}^{n},
	\end{equation}
	where $\frac{\partial f}{\partial {\bq}^{\mu}}=
	\left[ \frac{\partial f}{\partial q_{1}^{\mu}}, \ldots, 
	\frac{\partial f}{\partial q_{n}^{\mu}}\right]$.
	For simplicity, we will write $\nabla_{\bq}f$ as $\nabla f$ in the subsequent sections.
\end{definition}

It can be seen that {$\nabla_{\boldsymbol{\bar{q}}} f(\bq)$} represents the direction of the steepest ascent of the scalar-valued function $f(\bq)$, indicating the direction of maximum rate of change, {see e.g., \cite{dxu}}.

\section{Quaternionic Reweighted Amplitude Flow }\label{p3}

The main goal of this section is to design the  Reweighted Amplitude Flow algorithm for quaternion-valued signals (QRAF),  which  extends the Real/Complex Amplitude Flow (RAF) algorithm originally introduced in \cite{wang_solving_2018}. We organize the algorithm retaining the structure of the original RAF on purpose, however, it should keep in mind that the quaternion framework  introduces {\em non-commutativity}. We consider the quaternionic Gaussian measurement ensemble where the 
entries of ${\bA}$ are i.i.d. drawn from 
$
		\mathscr{N}_{\mathbb{H}}=\frac{1}{2}\left(\mathscr{N}(0,1)
		+\mathscr{N}(0,1)i + \mathscr{N}(0,1)j +\mathscr{N}(0,1)k \right),
$
denoted by ${\bA}\sim \mathscr{N}_{\mathbb{H}}^{n\times d}$, yielding $\mathbb{E}(\bal_{k} \bal_{k}^{*} )={\bf I}_{d}$.

\subsection{Algorithm}

The QRAF algorithm  is based on    {\bf the amplitude flow model}:
\begin{equation} \label{m2}
	\min_{\bz\in \mathbb{H}^{\,d}} F(\bz)=\frac{1}{n}
	\sum_{j=1}^{n} \left(|\langle \bal_{j}, \bz\rangle|-\psi_{j} \right)^{2},
\end{equation}   
where $\langle \bal_{j}, \bz\rangle=\bal_{j}^{*}\,\bz$ represents the quaternionic inner product, and $\psi_{j}=|\langle \bal_{j}, \bx\rangle|$ denotes the modulus of the quaternionic inner product between the known design vector $\bal_j$ and the unknown solution $\bx$. We assume that $\|\bx\|=1$. QRAF  begins with a reweighted initialization procedure, and subsequently refines the initial estimate $\bz_{0}$ through a quaternion-based gradient descent.
\subsubsection{Quaternionic Weighted Maximal Correlation Initialization}
The importance of selecting an effective starting point is well-recognized for non-convex iterative algorithms in achieving  global optimum. QRAF is not an exception. The quaternionic weighted maximal correlation initialization consists of two steps. First, the norm of the true signal $\bx$ is  estimated easily as:
\begin{equation} \label{ns1}
	\frac{1}{n}\sum_{j=1}^{n}\psi_{j}^{2}=\frac{1}{n}\sum_{j=1}^{n}\left|\langle \bal_{j}, \|\bx\|\boldsymbol{e}_{1} \rangle\right|^{2}
	\approx \|\bx\|^{2}.
\end{equation}
Next, the direction of the quaternion signal ${\bx}$ is estimated using a flexible weighting regularization technique  that balances the informative content derived from the selected data. As in the Real/Complex cases,  larger $\psi_i$ values suggest stronger a correlation between $\bal_i$ and $\bx$, indicating that  $\bal_i$ contains valuable directional information about $\bx$. 

More precisely, we sort the correlation coefficients $\{\psi_j\}_{1\le j\le n}$ in ascending order as $0 < \psi_{[n]} \le \ldots \le \psi_{[2]} \le \psi_{[1]}$. Let $\mathcal{S} \subset \mathcal{M}$ denote the set of selected feature vectors $\bal_j$  used for the initialization. The cardinality $|\mathcal{S}|$ is pre-defined as an integer on the order of $n$, e.g., $|\mathcal{S}| := \lfloor 3n/13\rfloor$. The set $\mathcal{S}$  is defined as the set  of $\bal_j$ vectors corresponding to the largest $|\mathcal{S}|$ correlation coefficients ${\psi_{[j]}}_{1\le j\le |\mathcal{S}|}$, each  approximately indicating in the direction of $\bx$. 
It is reasonable to assume that if $\psi_{i}$ is larger than $\psi_{j}$, then $\bal_{i}$ is more correlated with $\bx$ than $\bal_{j}$ is, hence providing more useful information regarding  the true direction of $\bx$. This motivates the assignment of higher weights to the selected $\bal_{i}$ vectors corresponding to larger $\psi_{i}$ values.

Approximating the direction of $\bx$ thus reduces to find a vector that maximizes its correlation with the subset $\mathcal{S}$
of selected directional vectors $\bal_j$. The desired approximation vector is again efficiently found by solving: 
\begin{equation*}
	\begin{split}
		\max_{\|\bz\|=1}\frac{1}{|\mathcal{S}|} \sum_{j\in \mathcal{S}}\omega_{j}^{(0)}|\langle \bal_{j}, \bz\rangle|^{2}
		= \bz^{*}\left(\frac{1}{|\mathcal{S}|} \sum_{j\in \mathcal{S}}\omega_{j}^{(0)} \bal_{j}\bal_{j}^{*} \right)\bz.
	\end{split}
\end{equation*}
Here $\omega_{j}^{(0)}:=\psi_{j}^{\gamma}$, and $\gamma$ is a carefully chosen parameter. By default, we set $\gamma=1/2$ in the reported numerical implementations.  This, combined with the norm estimate \eqref{ns1} to match the magnitude of $\bx$,  provides the initialization.

\subsubsection{Quaternionic Adaptively Reweighted Gradient Flow} 

The guessed initialization vector $\bz_{0}$ is refined by an adaptively  reweighted gradient descent.  The  new weighted quaternionic gradient  is given by
\begin{equation}\label{mo1}
	\nabla \ell_{\rm rw} (\bz) = \frac{1}{n} \sum_{k = 1}^{n} \omega_{k}^{(t)} \left( 1 - \frac{\psi_{k}}{\ABS{\bal_{k}^{*} \bz}}  \right) \bal_{k}\bal_{k}^{*}\bz,
\end{equation}
where the adaptive  weights are defined by $\omega_{k}^{(t)}=1/\left(1+\beta/\left(|\bal_{k}^{*}\bz|/|\bal_{k}^{*}\bx|\right)\right)$, for $1\le k\le n,$
in which the dependence on the iterate index $t$ is ignored for notational brevity.
It should be keep in mind the multiplication order in the factor $\bal_{k}\bal_{k}^{*}\bz$ in \eqref{mo1} is crucial, otherwise the algorithm may fail.

The idea behind  introducing this quaternionic gradient is as same as it in the  conventional RAF in \cite{wang_solving_2018}, which is to differentiate  the contributions of various gradients  to the overall search direction. A straightforward approach  is thus to assign large weights to more reliable gradients and smaller weights to the less reliable ones.  Therefore, the gradient is designed based on the ratio  $\left|\bal_{i}^{*}\bz /\bal_{i}^{*}\bx\right|$, which serves as a confidence score reflecting the reliability of  the corresponding gradient and potentially indicating directions that lead to the true $\bx$.
Taking a suitable step size $\eta$, the update rule is given by 
$\bz_{i+1}=\bz_{i} - \eta \cdot \nabla \ell_{\rm rw}(\bz_{i})$.

In summary, the QRAF algorithm, combining the initialization and gradient flow processes,  is outlined in Algorithm \ref{Alg_QRAF}.

\begin{algorithm}[!ht]
	\caption{Quaternionic Reweighted Amplitude Flow (QRAF)}
	\label{Alg_QRAF}
	\begin{algorithmic}[1] 
		\REQUIRE Data $\left(\bal_{k}, \psi_{k}=\ABS{\bal_{k}^{*}\bx}\right)_{k=1}^{n}$, step size $\eta$, weighting parameters $\beta$, subset cardinality $\ABS{\mathcal{S}}$, exponent $\gamma$ and the iteration number $T$; 
		\STATE Let the set $\mathcal{S}$ contains the indices of $\ABS{\mathcal{S}}$ largest entries in $\psi_{k}$, $k=1,2,\ldots n$.
		\STATE Construct the quaternionic Hermitian matrix
		\begin{equation*}
			\boldsymbol{S}_{in} = \frac{1}{n} \sum_{k = 1}^{n} \omega_{k}^{(0)} \bal_{k}\bal_{k}^{*}, 
		\end{equation*}
		where
		\begin{displaymath}
			\omega_{k}^{(0)} = \begin{cases}
				\psi_{k}^{\gamma},\quad & k \in  \mathcal{S} \subset \mathcal{M},\\
				0,\quad & \text{otherwise},
			\end{cases} 
		\end{displaymath}
		and find its normalized eigenvector $\boldsymbol{\nu}_{in}$ regarding its largest standard eigenvalue.
		\STATE Compute $\lambda_{0} = \left(\frac{1}{n}\sum_{k = 1}^{n} \psi_{k}^2\right)^{1/2}$ and obtain the spectral initialization $\bz_{0} = \lambda_{0}  \cdot \boldsymbol{\nu}_{in}$.
		\FORALL {$i = 0,1,\ldots T-1$}
		\STATE Compute
		\begin{equation} \label{eq_nabula_qraf}
			\nabla \ell_{\rm rw}(\bz_{i}) = \frac{1}{n} \sum_{k = 1}^{n} \omega_{k}^{(t)} \left( 1 - \frac{\psi_{k}}{\ABS{\bal_{k}^{*} \bz_{i}}}  \right) \bal_{k}\bal_{k}^{*}\bz_{i},
		\end{equation}
		where $\omega_{k}^{(t)} = \frac{\ABS{\bal_{k}^{*} \bz_{i}} / \psi_{k}}{\ABS{\bal_{k}^{*} \bz_{i}}^{2}/\psi_{k}+\beta}  $.
		\STATE Update $\bz_{i+1} = \bz_{i} - \eta \cdot \nabla \ell_{\rm rw}(\bz_{i}).$
		\ENDFOR
		
		\ENSURE $\bz_{T}$; 
	\end{algorithmic}
\end{algorithm}

\subsection{Linear Convergence Discussion}
In this section, we present a partial analysis of the convergence of the QRAF algorithm. The primary challenge  remains in establishing the local regularity condition, which is analogous to the difficulty encountered in the complex RAF  algorithm.  However, once this problem is resolved in the complex domain, we believe that extending the result to the quaternionic setting will not require significant additional effort.

We adopt the  distance defined in \cite{chen_phase_2023} between any two vectors $\bz$ and $\bx$ in $\mathbb{H}^{d}$ as follows 
\begin{equation}\label{ds}
	{\rm dist}(\bz, \bx)=\min_{|w|=1}\|\bz-\bx w\|,
\end{equation}
where $w\in \mathbb{H}$ accounts for the trivial ambiguity of the right quaternion phase factor. Recall that the minimum in the above equation \eqref{ds} is attained at $w={\rm sign}(\bx^{*}\bz)$,  yielding the expression
${\rm dist}(z, x)=\|\bz-\bx \,{\rm sign}(\bx^{*}\bz)\|$,
where ${\rm sign}(w):=w/|w|$ for nonzero $w\in \mathbb{H}$ and $ {\rm sign}(0):=1$.

For the initial guess $\bz_{0}$   by  spectrum method in Algorithm \ref{Alg_QRAF}, it is found that the weighted initialization is effective, following a similar discussion in the complex case. Hence we omit the detailed proof. 

\begin{proposition}[Weighted Initialization]
	For an arbitrary $x\in \mathbb{H}^{d}$, consider the noiseless measurements $\psi_{i}=|\psi^{*}x|$, $1\le i\le n$. If $n\ge c_{0}|S|\ge c_{1}d$, then with 
	probability   exceeding $1-c_{3}e^{-c_{2}n}$, the initial guess $z_{0}$ obtained by the weighted maximal correlation method satisfies 
	${\rm dist} (z_{0}, x)\le \rho\|x\|$
	for $\rho=1/10$. Here $c_{0}, c_{1}, c_{2}, c_{3}>0$
	are some absolute constants.  
\end{proposition}

Next, we are expecting to prove that starting from such an initial estimate, the iterates (in Step 4 of Algorithm \ref{Alg_QRAF}) converge at a linear rate to the global optimum $\bx$.  To achieve this, it suffices to show that the iterative update of QRAF is locally contractive within a relatively small neighboring 
of the true signal $\bx$. Thus once the initialization falls within  this neighborhood, linear convergence can be ensured with an appropriate choice of the constant step size. The local error contraction, and consequently  linear convergence,  directly follow  from the Local Regularity Condition (LRC) in a standard way.

\begin{definition}[Local Regularity  Condition] The reweighted gradient $\nabla \ell_{{\rm rw}}(\bz)$ is said to satisfy the local regularity  condition for positive parameters $\mu, \lambda, \varepsilon$, denoted as LRC $(\mu, \lambda, \varepsilon)$, if   
	\begin{equation*}
		\begin{split}
			&{\rm Re} \left\langle \nabla \ell_{{\rm rw}}(\bz), \bz-\bx \phi(\bz)\right\rangle  \ge \frac{\lambda}{2} {\rm dist}^{2}(\bz, \bx) 
			+\frac{\mu}{2} \|\nabla \ell_{{\rm rw}}(\bz)\|^{2}
		\end{split}
	\end{equation*}
	holds for all $\bz\in \mathbb{H}^{d}$ such that $\|\bz-\bx \phi(\bz)\|\le \varepsilon\|\bx\|$ for some constant $0<\varepsilon<1$. The ball given by $\|\bz-\bx \phi(\bz)\|\le \varepsilon\|\bx\|$ is termed the {\em basin of attraction} in literature.
\end{definition}

\begin{lemma}[Local error contraction] For an arbitrary $\bx\in \mathbb{H}^{d}$, consider $n$ noise-free measurements $\psi=|\bal_{j}^{*}\bx|, 1\le j\le n$. There exist  some constants $c_{1}, c_{2}, c_{3}>0$, and $0<\nu <1$  such that the following holds with probability exceeding $1-c_{3}e^{-c_{2}n}$,
	\begin{equation*}
		{\rm dist}^{2}(\bz_{t+1}, \bx)\le
		\left(1-\nu\right) {\rm dist}^{2}(\bz, \bx)
	\end{equation*}
	for all $\bx, \bz\in \mathbb{H}^{d}$ satisfying ${\rm dist }(\bz, \bx)\le \frac{1}{10}\|\bx\|$, provided that $n\ge c_{1}d$ and the constant step size $\mu \le \mu_{0}$, where the numerical constant $\mu_{0}$ depends on the parameter $\beta>0$ and data $\{(\bal_{i}; \psi_{i})\}_{1\le i\le n}$.
\end{lemma}
\begin{proof} Straightforward computations yield
	\begin{equation*}
		\begin{split}
			{\rm dist}^{2} (\bz-\mu \cdot \nabla \ell_{{\rm rw}}(\bz), \bx) \le \,&
			\|\bz-\mu \nabla \ell_{{\rm rw}}(\bz)-\bx\phi(\bz) \|^{2}\\
			= \, & {\rm dist}^{2}(\bz, \bx)+\mu^{2}\|\nabla \ell_{{\rm rw}}(\bz)\|^{2}
			-2\mu\cdot {\rm Re}\left\langle \nabla \ell_{{\rm rw}}(\bz), \bz-\bx\phi(\bz)\right\rangle\\
			\le \,&{\rm dist}^{2}(\bz, \bx)+\mu^{2}\|\nabla \ell_{{\rm rw}}(\bz)\|^{2}-2\mu \left(\frac{\mu}{2}\|\nabla \ell_{{\rm rw}}(\bz)\|^{2}+\frac{\lambda}{2}{\rm dist}^{2}(\bz, \bx)\right)\\ 
			\le \,&\left(1-\lambda\mu\right) {\rm dist}^{2}(\bz, \bx),
		\end{split}
	\end{equation*}
	where we used the LRC in the third step.
\end{proof}

Now, the main concern reduces to prove that within the neighborhood of the global minimizer, QRAF satisfies the LRC.

\begin{lemma}
	The reweighted gradient $\nabla \ell_{{\rm rw}}(\bz)$ satisfies  {\rm LRC}$(\mu, \lambda, \varepsilon)$.
\end{lemma} 
\begin{proof} (I) First we show that  \begin{equation} \label{bb1}  \|\nabla \ell_{{\rm rw}}(\bz)\| \le (1+\delta)\|\bz-\bx\phi(\bz)\|
	\end{equation}
	holds with high probability. As in the conventional case, rewrite the reweighted gradient in a compact 
	matrix-vector form 
	\begin{equation*}
		\begin{split}
			\nabla \ell_{{\rm rw}}(\bz)&=\frac{1}{n} \sum_{k = 1}^{n} \omega_{k}^{t} \left( 1 - \frac{y_{k}}{\ABS{\bal_{k}^{*} \bz}}  \right) \bal_{k}\bal_{k}^{*}z=\frac{1}{n}{\rm diag}(\boldsymbol{w})\bA\bv, 
		\end{split}
	\end{equation*}
	where ${\rm diag}(\boldsymbol{w})\in \mathbb{R}^{n\times n}$ is a diagonal matrix holding in order the entries of $\boldsymbol{w}=[\omega_{1}, \cdots, \omega_{n}]^{*}\in \mathbb{R}^{m}$
	on its main diagonal and $\boldsymbol{v}:=[v_{1}, \cdots, v_{n}]^{*}\in \mathbb{R}^{n}$
	with $v_{i}=\left( 1 - \frac{y_{k}}{\ABS{\bal_{k}^{*} \bz}}  \right)\bal_{i}^{*}\bz$.
	It follows that 
	\begin{equation*}\begin{split}
			\|\nabla\ell_{{\rm rw}}(\bz)\|&=\left\|\frac{1}{n}{\rm diag}(\boldsymbol{w})\bA\bv \right\|\le \frac{1}{n}
			\|{\rm diag}(\boldsymbol{w})\|\cdot\|\bA\|\cdot\|\bv\|\le \frac{1+\delta_{1}}{\sqrt{n}}\|\bv\|,
		\end{split}
	\end{equation*}
	where we have used the inequalities $\|{\rm diag}(\boldsymbol{w})\|\le 1$ since $\omega_{i}\le 1$
	for all $1\le i\le n$ and $\|\bA\|\le (1+\delta_{1})\sqrt{n}$ for some constants $\delta_{1}>0$, assuming that $n/d$ is sufficiently large. 
	
	Next, we bound $\|\bv\|$,
	\begin{equation*}\begin{split}
			\|\bv\|^{2} &\le \sum_{k=1}^{n}(|\bal_{i}^{*}\bz|-|\bal_{i}^{*}\bx|)^{2}
			\le \sum_{k=1}^{n}\|\bal_{i}^{*}(\bz-\bx\phi(\bz)) \|^{2} \le (1+\delta_{2})^{2}n\|\bz-\bx\phi(\bz)\|^{2}, 
		\end{split}
	\end{equation*}
	for some constant $\delta_{2}>0$, which holds with probability at least $1-e^{c_{2}n}$ as long as $n>c_{1}d$.
	
	Combining these results, and taking $\delta>0$  larger than $(1+\delta_{1})(1+\delta_{2})-1$, the size of $\nabla \ell_{{\rm rw}}(\bz)$ can be bounded as 
	$\|\nabla \ell_{{\rm rw}}(\bz)\|\le (1+\delta)\|\bz-\bx\phi(\bz)\|,$
	which holds with probability $1-e^{-c_{2}n}$,  
	with a proviso that $n/d$ exceeds some numerical constant $c>0$.
	
	(II) If we can show that for all $\bz$ satisfying $\|\bz-\bx\phi(\bz)\|\le c\|\bx\|$, the following holds
	\begin{equation} \label{bb2}
		{\rm Re} \left\langle \nabla \ell_{{\rm rw}}(\bz), \bz-\bx \phi(\bz)\right\rangle \ge c_{g}\|\bz-\bx\phi(\bz) \|^{2}
	\end{equation}
	with high probability, where $c_{g}>0$ is a constant. 

	Combining the two bounds \eqref{bb1} and \eqref{bb2},  we find that the LRC holds for $\mu$ and $\lambda$ satisfying 
	for sufficiently small $\varepsilon$ and $\delta$.
\end{proof}

Finally, we still need to prove the following, which is open in the complex case as well. 
\begin{pro}[{\bf Open}] \label{prop1} If there exists a numerical constant $c>0$ such that along the search direction $\nabla \ell_{{\rm rw}}(\bz)$, the following uniform lower bound holds
	\begin{equation} \label{pop1}
		{\rm Re}   \langle \nabla \ell_{{\rm rw}}(\bz), \boldsymbol{h}\rangle \ge c\|\boldsymbol{h}\|^{2}
	\end{equation}  
	for $\|\boldsymbol{h}\|\le 1/10 \|\bz\|$?
\end{pro}
{Such a bound, established  in the real number setting, relies heavily on the fact for $j=1, \cdots, n$, the value of $ \text{sign}\langle \bal_j, \bz\rangle$ is always either $-1$ or $1$, which can only  hold in the real number field. A straightforward generalization of proof to the complex domain is very challenging, as realized by several researchers, see e.g.,\cite{gswx} and \cite{10301360}. However, once resolved in  complex setting, the problem can  be further extended to the present quaternionic case using existing  techniques in quaternionic analysis \cite{dss}. Alternatively, to circumvent this difficulty, perturbed amplitude flow model generalizing \cite{gswx} is introduced in Section \ref{p5}. Within this framework,  the desired lower bound \eqref{pop1} can be established  rigorously. Finally, we believe that the global convergence is ensured by a similar (but more subtle) favorable benign geometric landscape, as established  in \cite{sqw}.}

\section{Further Developments of QRAF}\label{p4}

In this section, we introduce three variants that further refine the QRAF algorithm introduced in Section \ref{p3}. At the end of this section, we present the Quaternionic Reshaped Wirtinger Flow (QRWF), an extension of the RWF algorithm in \cite{zhang_nonconvex_nodate}, while the latter gives  a significant improvement of the seminar work of  Wirtinger Flow \cite{candes_phase_2015_Theory}. Notably, both the Quaternionic Truncated Amplitude Flow (QTAF) in \cite{chen_phase_2023} and the newly defined QRWF can be considered as special cases of the QRAF algorithm,  by choosing specific weights. The efficacy of the algorithms proposed in this section will be evaluated through numerical testing.  

\subsection{Incremental Algorithm: QIRAF--the critical one}
By reweighting the objective function at each iteration, both RAF as well as QRAF can make the
gradient descent algorithm easier to converge to the global minimum. On the other side, in large-sample and online scenarios, stochastic algorithms are typically preferred due to their  faster convergence rates and lower memory requirements.

Recently, incremental algorithms, such as Incremental Reshaped Wirtinger Flow (IRWF, see e.g. \cite{zhang_nonconvex_nodate}) , have  been developed. This further motivates the development of incremental or stochastic  versions of QRAF, using mini-batches of measurements, referred to as incremental QRAF (QIRAF). The mini-batch QIRAF  algorithm empolys the same initialization procedure as  QRAF, and use a mini-batch of measurements for each gradient update. We describe it in Algorithm \ref{Alg_QIRAF1} below. 
\begin{algorithm}[!ht]
	\caption{Quaternionic Mini-batch Incremental Reweighted Amplitude Flow (QIRAF)}
	\label{Alg_QIRAF1}
	\begin{algorithmic}[1] 
		\REQUIRE Same as in QRAF. Moreover, the size of mini-batch $\ABS{\Gamma}$; 
        
		\STATE Same spectral initialization to obtain $\bz_{0} = \lambda_{0}  \cdot \boldsymbol{\nu}_{in}$ as in QRAF.
		\FORALL {$i = 0,1,\ldots T-1$}
		\STATE Uniformly select $\ABS{\Gamma}$ random number in $\{1,2,\ldots,n\}$ as the mini-batch set $\Gamma_{i}$ and then compute 
		\begin{equation*} 
			\nabla f(\bz_{i}) = \frac{1}{\ABS{\Gamma}} \sum_{k \in \Gamma_{i} } \omega_{k}^{(t)} \left( 1 - \frac{y_{k}}{\ABS{\bal_{k}^{*} \bz_{i}}}  \right) \bal_{k}\bal_{k}^{*}\bz_{i},
		\end{equation*}
		where $\omega_{k}^{(t)}$  is defined in (\ref{eq_nabula_qraf}).
		\STATE Update $\bz_{i+1} = \bz_{i} - \eta \cdot \nabla f(\bz_{i}).$
		\ENDFOR
		
		\ENSURE $\bz_{T}$; 
	\end{algorithmic}
\end{algorithm}
It is important to note that in \cite{wang_phase_2018}, the mean is recommended as $\frac{1}{d}$ rather than $\frac{1}{\ABS{\Gamma}}$, which is employed in this study. Through rigorous  analysis and testing, it is found that $\frac{1}{\ABS{\Gamma}}$ is a more appropriate coefficient for the quaternionic gradient of the algorithm. Addtionally, the batch size of $\ABS{\Gamma} = 64$ recommended in \cite{wang_phase_2018} is broadly applicable; however, such value of parameter is not optimal for quaternionic cases. Here we  recommended a batch size of $\ABS{\Gamma} = 2^{k}$, $k = \frac{\log\left(\frac{n}{4}-1\right)}{\log 2}$,  which is actually the minimal value of $2^{k}>\frac{n}{4}-1$.

Comparative analyses presented in Section \ref{p4} show that QIRAF exhibits good statistical and computational performance. Furthermore, QIRAF has a lower sampling complexity than QIRWF, hence provides a more efficient option in practical applications.

\subsection{Accelerated Algorithm: QARAF} \label{sub_sec_QARAF}

WF and QWF algorithms do not require manual parameter adjustment; however, their convergence rates are relatively slow. For real and complex signals, several accelerated steepest gradient methods have been proposed, demonstrating practical effectiveness (see, e.g., \cite{qin}). Convergence guarantees for these accelerated first-order methods are established in \cite{xch}.

Following the line of research, we introduce an acceleration scheme for  QRAF.  The initialization of $\bz_{0}$ is selected  as in  the QRAF algorithm. This   estimate is then refined   iteratively by applying an accelerated steepest decent method to the QRAF update rule.  The  accelerated reweighted iterative procedure inductively is formally defined as follows,
\begin{equation} \label{ac1}
	\begin{cases}
		\bz_{i+1} = \boldsymbol{\psi}_{i} - \eta \cdot \nabla f(\boldsymbol{\psi}_{i}),\\
		\boldsymbol{\psi}_{i+1} = \bz_{i+1} + \mu \left(\bz_{i+1} - \bz_{i} \right),
	\end{cases}
\end{equation}
where $i=1, 2, \ldots, T$, $\eta>0$ is the step size which is suggested as $\eta=6$, and $\boldsymbol{\psi}_{0} = \bz_{0}$.
This algorithm is  termed as Quaternionic Accelerated Reweighted  Flow (QARAF) when the extrapolation parameter $\mu=0.8$ in Equ.\,\eqref{ac1} is chosen as  suggested by Nesterov's method.
The QARAF algorithm is presented in Algorithm \ref{Alg_QARAF}.  Comparative analyses in  Fig. \ref{Pic_Converg_compare} and \ref{Pic_Converg_compare_PQ} show that QARAF converges significantly faster than QWF, QTAF, QRWF and QRAF. Fig.\,\ref{Pic_success_rate_A} shows that QARAF also has better performance on the success rates than QWF, QTAF and QRAF.
\begin{algorithm}[!ht]
	\caption{Quaternionic Accelerated Reweighted Amplitude Flow (QARAF)}
	\label{Alg_QARAF}
	\begin{algorithmic}[1] 
		\REQUIRE Same as in QRAF. Moreover, set the accelerator parameter $\mu$.
		
		\STATE Same spectral initialization to obtain $\bz_{0} = \lambda_{0}  \cdot \boldsymbol{\nu}_{in}$ as in QRAF.
		\FORALL {$i = 0,1,\ldots T-1$}
		\STATE Compute $\nabla f(\bz_{i})$ as in QRAF in (\ref{eq_nabula_qraf}).
		\STATE Let $\boldsymbol{\psi}_{0} = \bz_{0}$, then update 
		\begin{displaymath}
			\begin{cases}
				\bz_{i+1} = \boldsymbol{\psi}_{i} - \eta \cdot \nabla f(\boldsymbol{\psi}_{i}),\\
				\boldsymbol{\psi}_{i+1} = \bz_{i+1} + \mu \left(\bz_{i+1} - \bz_{i} \right).
			\end{cases} 
		\end{displaymath}
		\ENDFOR
		
		\ENSURE $\bz_{T}$; 
	\end{algorithmic}
\end{algorithm}

\subsection{Adapted Algorithm: QAdRAF}

In this subsection, we adopt another steepest gradient scheme, i.e. the adapted gradient decent method,  in the updated rule of the QRAF algorithm, refered to as Quaternion Adaptive Reweighted  Amplitude Flow (QAdRAF).  As one can see in 
Fig.\,\ref{Pic_Converg_compare} and \ref{Pic_Converg_compare_PQ},  it preforms the second fast convergence rate among those quaternion algorithms. QAdRAF also has nice success rate performance, see Fig.\,\ref{Pic_success_rate_B}. This algorithm is detailed in Algorithm \ref{Alg_QAdRAF}.

\begin{algorithm}[!ht]
	\caption{Quaternionic Adaptive Reweighted Amplitude Flow (QAdRAF)}
	\label{Alg_QAdRAF}
	\begin{algorithmic}[1] 
		\REQUIRE Same as in QRAF. Moreover, set the coefficient $\mu \in \left(0,1\right)$, parameters $\epsilon$ and {$\alpha$};
		
		\STATE Same spectral initialization to obtain $\bz_{0} = \lambda_{0}  \cdot \boldsymbol{\nu}_{in}$ as in QRAF.
		\FORALL {$i = 0,1,\ldots T-1$}
		\STATE Compute $\nabla f(\bz_{i})$ as in QRAF in (\ref{eq_nabula_qraf}).
		\STATE Let $S_{0} = 0$ and $S_{i+1} = \mu S_{i} + \left(1-\mu\right)\|\nabla f(\bz_{i})\|^{2}$, then we compute
		\begin{displaymath}
			\eta = \frac{\alpha}{\left(S_{i+1}+\epsilon\right)^{1/2}},
		\end{displaymath}
		\STATE Update $\bz_{i+1} = \bz_{i} - \eta \cdot \nabla f(\bz_{i}).$
		\ENDFOR
		
		\ENSURE $\bz_{T}$; 
	\end{algorithmic}
\end{algorithm}

\subsection{QRWF Algorithm}
The Reshaped Wirtinger Flow (RWF), as introduced in \cite{zhang_nonconvex_nodate},  represents a significant extension of the  seminar work of  Wirtinger Flow in \cite{candes_phase_2015_Theory}. It has nice performance but without complicated truncated procedure as TAF in \cite{wang_solving_2018}. Here, we introduce the Quaternionic Reshaped Wirtinger Flow (QRWF) in analogy.  The QRWF can be viewed as a special case of the QRAF by selecting an appropriate weight, which is summarized in Algorithm \ref{Alg_QRWF}.
\begin{algorithm}[!ht]
	\caption{Quaternionic Reshaped Wirtinger Flow (QRWF)}
	\label{Alg_QRWF}
	\begin{algorithmic}[1] 
		\REQUIRE $\left(\bal_{k},y_{k}=\ABS{\bal_{k}^{*}x}\right)_{k=1}^{n}$, set step size $\eta$, lower and upper thresholds $\alpha_{\ell}$ and $\alpha_{u}$  and iteration number $T$;
		
		\STATE Compute $\lambda_{0} = \frac{nd}{\sum_{k=1}^{n}\NORM{\bal_{k}}_{1}} \left(\frac{1}{n}\sum_{k = 1}^{n} y_{k}\right)$.
		\STATE Construct the data matrix
		\begin{equation*}
			\boldsymbol{S}_{in} = \frac{1}{n} \sum_{k = 1}^{n} y_{k} \bal_{k}\bal_{k}^{*} \mathbf{1}_{\{\alpha_{\ell}\lambda_{0} < y_{k} < \alpha_{u}\lambda_{0}\}},
		\end{equation*}
		where
		\begin{displaymath}
			\mathbf{1}_{\{\alpha_{\ell}\lambda_{0} < y_{k} < \alpha_{u}\lambda_{0}\}}  = \begin{cases}
				1,\quad & \alpha_{\ell}\lambda_{0} < y_{k} < \alpha_{u}\lambda_{0},\\
				0,\quad & \text{otherwise},
			\end{cases} 
		\end{displaymath}
		and find the normalized eigenvector $\boldsymbol{\nu}_{in}$ of $\boldsymbol{S}_{in}$ regarding its largest standard eigenvalue.
		\STATE The spectral initialization can be obtained as $\bz_{0} = \lambda_{0}  \cdot \boldsymbol{\nu}_{in}$.
		\FORALL {$i = 0,1,\ldots T-1$}
		\STATE Compute 
		\begin{displaymath}
			\nabla f(\bz_{i}) = \frac{1}{n} \sum_{k = 1}^{n} \left( 1 - \frac{y_{k}}{\ABS{\bal_{k}^{*} \bz_{i}}}  \right) \bal_{k}\bal_{k}^{*}\bz_{i},
		\end{displaymath}
		\STATE Update $\bz_{i+1} = \bz_{i} - \eta \cdot \nabla f(\bz_{i}).$
		\ENDFOR
		
		\ENSURE $\bz_{T}$; 
	\end{algorithmic}
\end{algorithm}

\section{Quaternionic Perturbed Amplitude Flow (QPAF)}\label{p5}
The TAF, TWF, and RAF algorithms for complex signals, as previously discussed, lack a comprehensive theoretical analysis. To address this issue, a new model was introduced
in \cite{gswx}, defined as
\begin{equation}\label{paf1}
	\min_{\bz} f_{\varepsilon}(\bz)=\min_{\bz}\frac{1}{n}\sum_{j=1}^{n}\left(\sqrt{|\bal_{j}^{*}\bz|^{2}+{\varepsilon_{j}^{2}}}-\sqrt{b_{j}^{2}+\varepsilon_{j}^{2}} \right)^{2},
\end{equation}
where $\boldsymbol{\varepsilon}=[\varepsilon_{1}, \cdots, \varepsilon_{n}]\in \mathbb{R}^{n}$ is a vector with prescribed values, subject to the condition that $\varepsilon_{j}\neq 0$ for all $b_{j}\neq 0$. When  $\varepsilon_{j}=0$ for all $j=1,\cdots, n$, {this model reduces to the complex amplitude-based model in \cite{wang_solving_2018}}, and is therefore referred to as the perturbed amplitude-based model. Based on this formulation, a new non-convex algorithm, called the perturbed Amplitude Flow (PAF) is developed. The PAF algorithm can recover the target signal under $\mathcal{O}(d)$ Gaussian random measurements. Starting from  a well-designed initial point,  PAF converges at a linear rate for both real and complex signals, without requiring  truncation or reweighted procedures, while maintaining strong empirical performance. 

In this section, we extend this framework to  the Quaternionic Perturbed Amplitude Flow (QPAF) algorithm, of course utilizing  the generalized $\mathbb{HR}$ calculus.  Numerical experiments show that QPAF performs   comparably to  PRAF, and significantly outperforms  QTWF and QWF (see Section \ref{p4}). Crucially, the  primary advantage of the model \eqref{paf1}  is preserved when  generalized to the quaternionic setting. Specifically,  with an appropriate choice of $\varepsilon$,  the  gradient $\nabla f_{\boldsymbol{\varepsilon}}(\bz_{i})$ can be effectively controlled in the neighbourhood of the initial guess. Moreover,  the gradient $\nabla f_{\boldsymbol{\varepsilon}}(\bz_{i})$  satisfies  {\rm LRC}$(\mu, \lambda, \varepsilon)$, addressing the primary theoretical limitation of QRAF, see Problem \ref{prop1}. Consequently, QPAF has a linear convergence rate. Since the proof closely mirrors the original,  except with  special attention given to the order of quaternion multiplication,  the full details are omitted here for brevity. The QPAF algorithm is provided in the following Algorithm \ref{Alg_QPAF}.

\begin{algorithm}[!ht]
	\caption{Quaternionic Perturbed Amplitude Flow (QPAF)}
	\label{Alg_QPAF}
	\begin{algorithmic}[1] 
		\REQUIRE Data $\left(\bal_{k}, \psi_{k}=\ABS{\bal_{k}^{*}\bx}\right)_{k=1}^{n}$, step size $\eta$, weighting parameter $\gamma$, the control coefficient $\sigma$ and the iteration number $T$;
        
		\STATE Compute $\lambda_{0} = \left(\frac{1}{n}\sum_{k = 1}^{n} \psi_{k}^2\right)^{1/2}$.
		\STATE Construct the quaternionic Hermitian matrix
		\begin{displaymath}
			\boldsymbol{S}_{in} = \frac{1}{n} \sum_{k = 1}^{n} \left(\gamma - e^{-\psi_{k}^2 / \lambda_{0}^{2}}\right)\bal_{k}\bal_{k}^{*}, 
		\end{displaymath}
		and find its normalized eigenvector $\boldsymbol{\nu}_{in}$ regarding its largest standard eigenvalue.
		\STATE Compute the initialization $\bz_{0} = \lambda_{0}  \cdot \boldsymbol{\nu}_{in}$..
		\FORALL {$i = 0,1,\ldots T-1$}
		\STATE Compute
		\begin{equation} \label{eq_nabula_qpaf}
			\nabla f_{\boldsymbol{\varepsilon}}(\bz_{i}) = \frac{1}{n} \sum_{k = 1}^{n} \left( 1 - \frac{\sqrt{\psi_{k}^{2} + \boldsymbol{\epsilon}_{k}^{2}}}{\sqrt{\ABS{\bal_{k}^{*} \bz_{i}}^2+ \boldsymbol{\epsilon}_{k}^{2}}}  \right) \bal_{k}\bal_{k}^{*}\bz_{i},
		\end{equation}
		where $\boldsymbol{\varepsilon} = \sqrt{\sigma} \boldsymbol{\psi}$, in which $\boldsymbol{\psi} = \{ \psi_{k}\}$, $k = 1,2,\ldots,n$.
		\STATE Update  $\bz_{i+1} = \bz_{i} - \eta \cdot \nabla f_{\boldsymbol{\varepsilon}}(\bz_{i}).$
		\ENDFOR
		
		\ENSURE $\bz_{T}$; 
	\end{algorithmic}
\end{algorithm}

\section{Algorithm regarding to RGB picture}\label{p6}

In quaternion  image processing, the color channels (Red, Green, and Blue channels) are usually represented  by the three imaginary components of quaternions, making the desired signal taking pure quaternions. Many studies (see e.g., \cite{cxz}) opt to eliminate the real part of the quaternion signals after  reconstruction. However, this approach  generally  performs poorly in the context of phase retrieval  for pure quaternion signals. A more effective approach is to incorporate the pur quaternion  assumption directly into the recovery process, as done in  \cite{chen_phase_2023}. In this section, we  brief  review  of the  Phase Factor Estimate (QPFE) technique  used in \cite{chen_phase_2023},  which will be applied to refine our algorithms introduced in previous sections and then applied to phase retrieval of natural images. 

Recall that for a quaternion number $q=q_{0}+q_{1} i+q_{2}j+q_{3}k$, we call $\mathcal{R}(q) = q_{0}$ the real part of $q$ and $\mathcal{P}^{\theta}(q)$, $\theta = i, j, k$ as the imaginary parts of $q$, i.e. $\mathcal{P}^{i}(q) = q_{1}$, $\mathcal{P}^{j}(q) = q_{2}$ and $\mathcal{P}^{k}(q) = q_{3}$. Let ${\bq}$ be a quaternion signal in $ \mathbb{H}^{d}$, the projection $\mathcal{V}$  splits this vector to its real and imaginary parts, i.e.  $\mathcal{V}({\bq}) = \left[ \mathcal{R}(\bq), \mathcal{P}^{i}(\bq), \mathcal{P}^{j}(\bq), \mathcal{P}^{k}(\bq)\right] \in \mR^{d\times4}$. Furthermore, when $q_{0} = 0$,  the quaternion number $q=q_{1} i+q_{2}j+q_{3}k$ is called by pure quaternion. Then, for a pure quaternion signal ${\bp}\in \mathbb{H}^{d}$, we have $\mathcal{V}({\bp}) = \left[ 0, \mathcal{P}^{i}(\bp), \mathcal{P}^{j}(\bp), \mathcal{P}^{k}(\bp)\right] \in \mR^{d\times 4}$.

The Phase Factor Estimate is not a complete algorithm for Phase retrieval but rather a simple yet useful technique that transforms a full quaternion signal to a pure quaternion  with  minimal difference. The strategy of QPFE is to find a quaternion phase factor $q$ with $\|q\|=1$ such that $\bz q$ is closest to pure quaternion signal, i.e., 
\begin{equation} \label{pf12}
	\hat{q}=\arg \min_{\|q\|=1}\| \textrm{Re} (\bz q) \|
\end{equation}
and then we map $\bz$ to the imaginary part of $\bz \hat{q}$. The above optimization problem \eqref{pf12} can be easily solved by spectrum method after transformed to an equivalent real matrix. The detailed process  is summarized in  Algorithm \ref{QPFE}. It is important to note that in Algorithm 3 in \cite{chen_phase_2023}, the $q_{(i+1)T_{p}}$  should be its quaternion conjugate $\overline{q_{(i+1)T_{p}}}$.

\begin{algorithm}[!ht]
	\caption{Quaternionic Phase Factor Estimate (QPFE)}
	\label{QPFE}
	\begin{algorithmic}[1] 
		\REQUIRE An array $\bz\in \mathbb{H}^{d}$;
		
		\STATE Split the full quaternionic array $\bz$ to a real matrix $\boldsymbol{M} \in \mR^{d \times 4}$.
		\STATE Compute the eigenvalue and eigenvector of $\boldsymbol{W} = \boldsymbol{M}^{T}\boldsymbol{M}$ which is a $4 \times 4$ matrix, obtain the eigenvector $\boldsymbol{v}$ of $\boldsymbol{W}$ corresponding to its smallest eigenvalue.
		\STATE Let $q = v_{1} +  v_{2} i +  v_{3}j +  v_{4}k$ be the phase factor and obtain its quaternionic conjugation $\bar{q} = v_{1} -  v_{2} i -  v_{3}j - v_{4}k$.
		\FORALL {$i = 0,1,\ldots T-1$}
		\STATE Compute the left  product of $\bz$ by $\bar{q}$, i.e., $\boldsymbol{p} = \bz \cdot \bar{q}.$
		\STATE Let $\bw$ be the imaginary part of $\boldsymbol{p}$.
		\ENDFOR
		
		\ENSURE $\bz_{T}$; 
	\end{algorithmic}
\end{algorithm}

Under suitable conditions, the pure quaternion signals can be exactly recovered as shown in \cite{chen_phase_2023}. Thus, in this case, it is more natural  to adopt the distance ${\rm dist}_{p}\left(\bw, \bz\right)$  to measure the reconstruction error of pure quaternion signals, defined as
\begin{equation} \label{dist_p}
	{\rm dist}_{p}\left(\bw, \bz\right) = \min\left\{ \NORM{\bz + \bw}, \NORM{\bz-\bw} \right\}.
\end{equation}
In the following experiments, we always use this distance to measure the pure quaternion signal unless stated otherwise.

At the end of this section, we provide a combination of QPFE and QRAF for pure quaternion signals, referred to as PQRAF. Combinations with other algorithms used in the Section \ref{p7} can be formulated similarly.

\begin{algorithm}[!ht]
	\caption{Pure Quaternionic Reweighted Amplitude Flow (PQRAF)}
	\label{Alg_PQRAF}
	\begin{algorithmic}[1] 
		\REQUIRE Data $\left(\bal_{k}, \psi_{k}=\ABS{\bal_{k}^{*}\bx}\right)_{k=1}^{n}$, step size $\eta$, weighting parameters $\beta$, subset cardinality $\ABS{\mathcal{S}}$, exponent $\gamma$ and the iteration number $T$. Moreover, the step parameter $T_{p}$;
		
		\STATE Initialization as in Algorithm (\ref{Alg_QRAF}) to obtain $\bz_{0} = \lambda_{0}  \cdot \boldsymbol{\nu}_{in}$.
		\FORALL {$i = 0,1,\ldots T-1$}
		\STATE Compute $\nabla \ell_{\rm rw}(\bz_{i})$ and update $\bz_{i+1}$ as in Algorithm \ref{Alg_QRAF}.
		\IF{$\mod(i,T_p) = 0$}
		\STATE Compute $\bw_{i+1} \in \mathbb{H}^{d}_{p}$ by calling QPFE as in Algorithm \ref{QPFE} with the input $\bz_{i+1}$.\\
		\STATE Update $\bz_{i+1} = \bw_{i+1}$.
		\ENDIF
		\ENDFOR
		
		\ENSURE $\bz_{T}$; 
	\end{algorithmic}
\end{algorithm}
Similar to \cite{chen_phase_2023}, we use $T_{p}$ to be the parameter that controls the transform of full quaternion to pure quaternion signals. Unlike \cite{chen_phase_2023} which utilizes inner and outer loop for the result, we only perform QPFE once after $T_{p}$ iteration periodically. In our algorithm, the total iteration number is the given parameter $T$ compare to the actual iteration number $T \times T_{p}$ in \cite{chen_phase_2023}. It is easily to see that the parameter $T_{p}$ makes difference when it changes, especially when $T_{p} = 1$ which means operate QPFE in every iteration and $T_{p} = T$ which means operate QPFE at the final step of the full quaternion algorithm. Although there are difference among them, the  actual computation shows no essential distinction. Since our work considers mainly on different algorithms rather than different aspects of a single algorithm, we prefer to save this part of work in the future.  

\section{Numerical Experiments}\label{p7}
In this section, we show the numerical efficiency of QRAF and its improvements by  comparing their  the performance with other competitive methods. 
Throughout our work, the parameters of the tested algorithms are listed in Table \ref{Table_parameters}. For the existing quaternionic algorithm QWF, we use the given parameters there. For the algorithms that transferred from the real or complex space, we compare the original parameters with a huge a mount of tested parameters to obtain a suitable set of parameters.

\begin{table}[!htbp]
	\caption{Parameters in algorithms}
	\centering
	\begin{tabular}{ll|ll}
		\toprule
		Algorithm & Parameters & Algorithm & Parameters \\
		\midrule
		QWF     & $ \eta = \frac{0.2n}{\sum_{k=1}^{n}y_{k}}$;&QRAF    & $\ABS{\mathcal{S}} = \lfloor 3 n / 13\rfloor$, $\beta = 5$, $\gamma =0.5$ and $\mu = 6$\\ 
		QRWF    & $\alpha_{\ell} = 1$, $\alpha_{u} = 5$ and $\mu = 0.8$;&QARAF   & $\eta = 6$ and $\mu = 0.8$;\\
		QPAF    & $\gamma = 1/2 $, $\eta = 2.5$ and $\sigma = 2$;&QAdRAF  & $\alpha = 0.009$, $\mu = 1$, $\epsilon = 10^{-6}$\\
		QTAF    & $\gamma = 0.8 $, $\eta = 1.2$ and $\rho = \frac{1}{6}$&QIRAF   & $\eta = 6$ and batch size $\ABS{\Gamma} = 2^{k}$, $k = \lfloor \frac{\log\left(\frac{n}{4}-1\right)}{\log 2}\rfloor$;\\
		\bottomrule
	\end{tabular}
	\label{Table_parameters}
\end{table}

\subsection{Synthetic Data} \label{sd1} This subsection presents a performance comparison of various quaternionic algorithms, evaluated through experimental results. In each trial, we employ a Gaussian measurement ensemble  $\bA\sim \mathscr{N}_{\mathbb{H}^{n\times d}}$. The entries of quaternion signal $\bx\in \mathbb{H}^{d}$  are independent and identically distributed (i.i.d.) as $\mathscr{N}(0, 1) +\sum_{v=i, j, k}\mathscr{N}(0, 1)v$. The signal $\bx$ is  subsequently normalized to satisfy $\|\bx\|= 1.$  Note that to approximate the leading eigenvector  of the Hermitian matrix, we perform 100 power iterations. The power method (see e.g., \cite{LWZZ}) used for quaternion Hermitian matrix  is analogous to the approach used in the complex case.

{\bf i. Convergence performance} \, Firstly, we examine the convergence of each algorithm, as shown in Fig.\,\ref{Pic_Converg_compare} and Fig.\,\ref{Pic_Converg_compare_PQ}. The relative error, defined as $\text{dist}(x,z)/\NORM{x}$ and $\text{dist}_{p}(x,z)/\NORM{x}$ w.r.t. full and pure quaternion algorithms respectively, serves as the convergence criterion.  It is seen that all  algorithms are convergent albeit at different rates. The convergence rate is a crucial  aspect  and  is influenced not only by the algorithm's gradient but also by other factors, such as the vector dimension $d$, the ratio $n/d$, and the algorithm's initialization.
\begin{figure}[htbp]
	\centering
	\begin{minipage}{0.45\textwidth}
		\centering
		\includegraphics[width=1.0\textwidth]{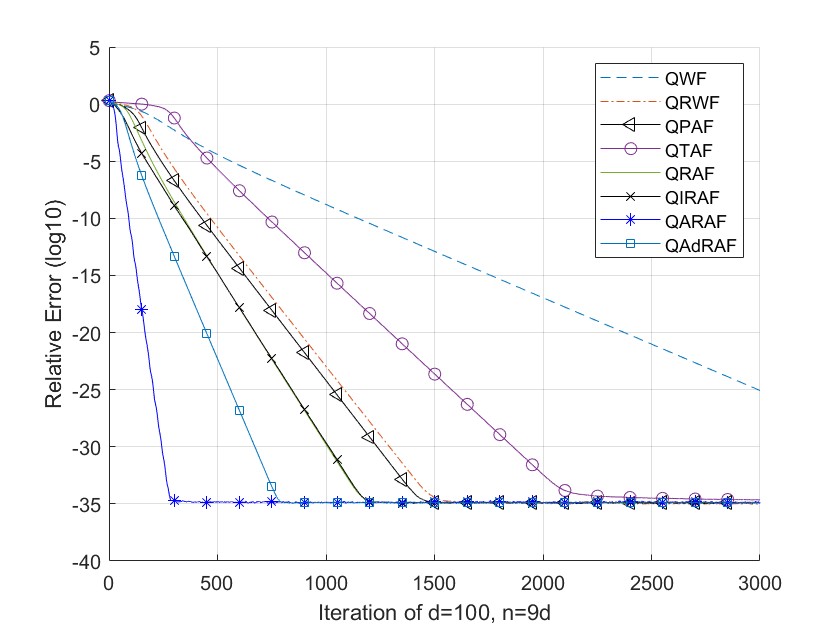}
		\caption{Simulations of different algorithms on quaternion-valued signals}
		\label{Pic_Converg_compare}
	\end{minipage}
	\hfill
	\begin{minipage}{0.45\textwidth}
		\centering
		\includegraphics[width=1.0\textwidth]{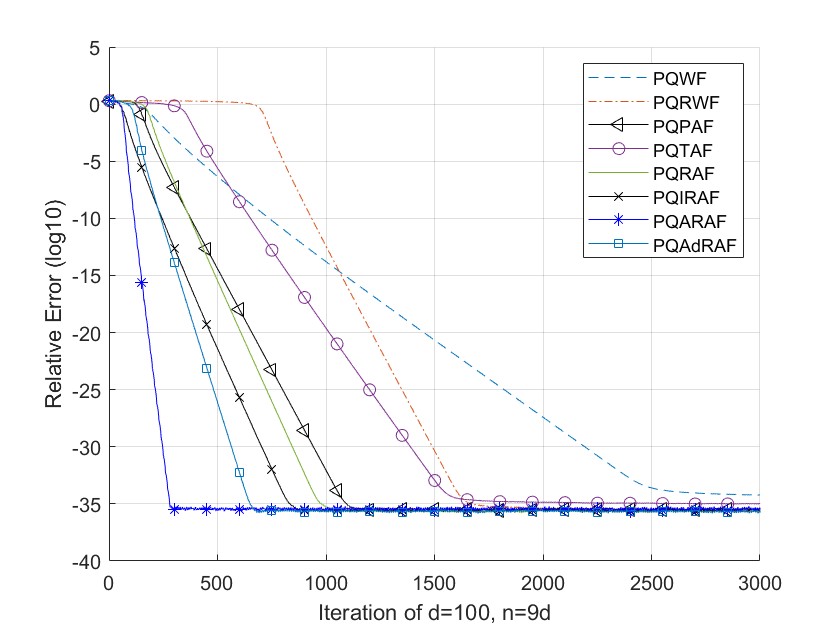}
		\caption{Simulations of different algorithms on signals of pure quaternions}
		\label{Pic_Converg_compare_PQ}
	\end{minipage}
\end{figure}

{We fix the oversampling rate to $n/d = 9$ as an example to present the convergence analysis.} It is easy to see that in general all algorithms converge to the same level of relative error, i.e. $\ln(10^{-35})$, but there are some details that must be noted. For the full quaternion algorithms in Fig.\,\ref{Pic_Converg_compare}, QARAF leads the competition of convergence with the iteration number less than $400$. QRAF and its variants converge evidently faster than other algorithms. The QWF also converges but with a lower speed that is not fully shown in the figure. We should note that the curves of QRAF and QIRAF almost cover each other, showing the same ability of convergence. This is actually important since as the variants of QRAF, QIRAF has the advantage of computation speed which we will discuss later in details.

Similar to the full quaternion algorithms, the pure ones have the same level and order of convergence. Note that with the assist of QPFE, the gradients vary on certain area, depending on the merit of algorithms themselves. But in general, the full quaternion algorithms and the pure ones admit the same level of performance. 

{\bf ii. Success rate}\, {We next evaluate the success rates of QRAF and its variants, along with other benchmark algorithms, under varying $n/d$ rate. The \textbf{oversampling rate $n/d$} is a critical factor in phase retrieval, as lower values imply reduced computational cost, faster execution, and lower memory consumption. Therefore, achieving successful recovery at smaller $n/d$ is a key indicator of algorithmic efficiency. }

{We conduct experiments on quaternion-valued signals $\bx \in \mathbb{H}^d$ with $d = 64$ and $d = 100$, respectively. For each configuration, we perform 100 trials over $n/d \in [[3, 13]]$ with increments of $0.5$. Noting that the interval $n/d \in (5,8)$ is particularly sensitive, we reduce the step size to $0.2$ within this range to capture more granular performance behavior. Each algorithm was run for 1500 iterations per trial. Similar as in \cite{chen_phase_2023, wang_phase_2018}, a trial was deemed successful if it achieved $\text{dist}(\bx,\bz)<10^{-5}$ within the 1500 iterations. The hyperparameter settings for each algorithm are listed in Table~\ref{Table_parameters}, and the resulting success rates are reported in  Fig. \ref{Pic_success_rate_A} and~\ref{Pic_success_rate_B}. The most notable performance is achieved by QARAF, which reaches a success rate of $1.0$ with an oversampling ratio below $n/d = 7$, clearly outperforming all other algorithms. This superior performance is attributed to the accelerated steepest descent strategy adopted in QARAF, as detailed in Section \ref{sub_sec_QARAF}.}

\begin{figure}[htbp]
	\centering
	\begin{minipage}{0.45\textwidth}
		\centering
		\includegraphics[width=1.0\textwidth]{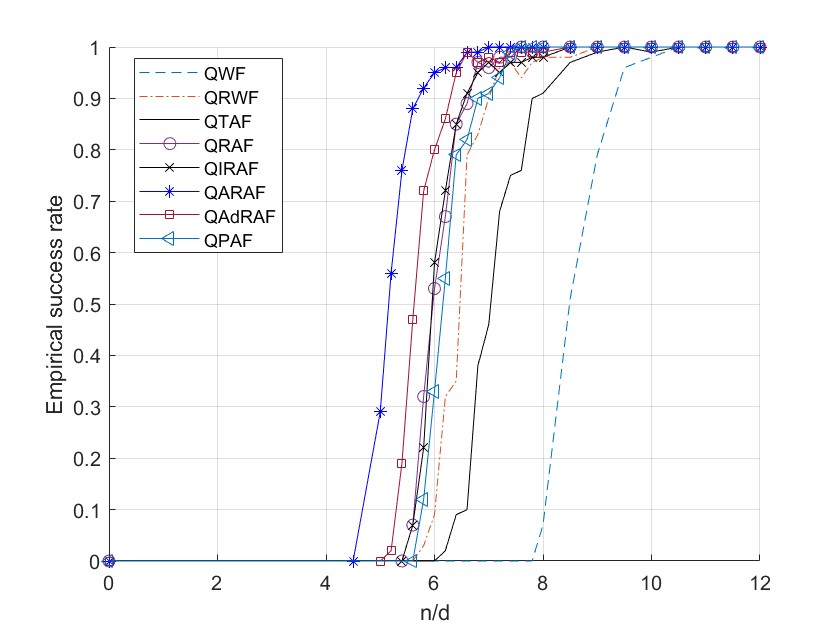}
		\caption{Success rate for $d = 64$}
		\label{Pic_success_rate_A}
	\end{minipage}
	\hfill
	\begin{minipage}{0.45\textwidth}
		\centering
		\includegraphics[width=1.0\textwidth]{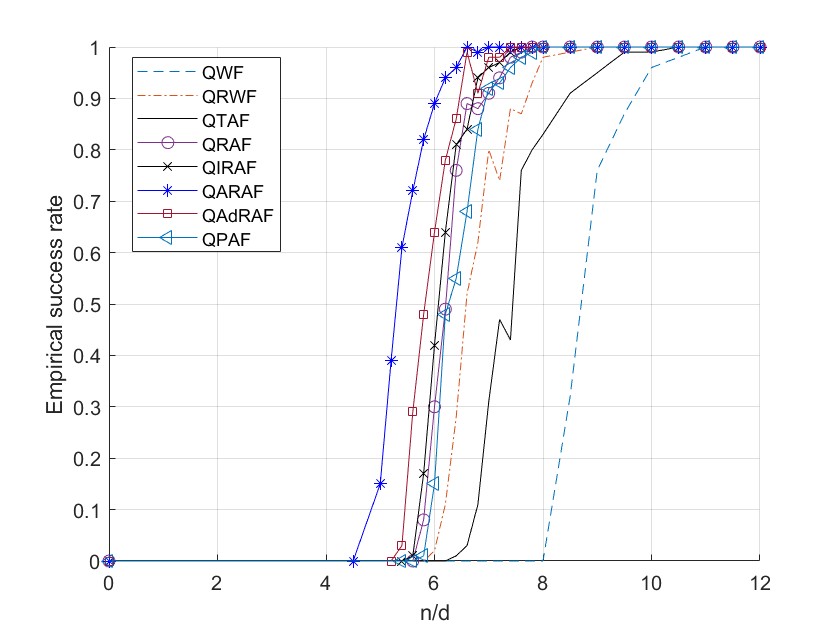}
		\caption{Success rate for $d = 100$}
		\label{Pic_success_rate_B}
	\end{minipage}
\end{figure}


QAdRAF, which incorporates an adaptive gradient descent mechanism, also demonstrates improved performance over both the baseline QRAF and other methods in the benchmark group shown in Fig.~\ref{Pic_success_rate_A}. Particularly noteworthy is the performance of QIRAF. While its success rate closely matches that of QRAF, it substantially reduces computational overhead by employing a mini-batch strategy. Despite this simplification, QIRAF retains nearly the same recovery capability, making it an attractive alternative for practical implementations. As shown in Table~\ref{t6a}, QIRAF achieves a comparable runtime to QARAF, thus offering an effective trade-off between accuracy and efficiency.

Two important remarks apply to the interpretation of success rate curves. First, since success is defined solely by reaching $\text{dist}(\bx,\bz)<10^{-5}$ within 1500 iterations, the curves do not distinguish between fast and slow convergence within a trial. Second, small fluctuations or inflection points in the curves (e.g., due to randomness in measurement matrices) are expected, especially when $n/d$ varies in small increments (e.g., 0.2). These do not affect the overall trends and conclusions, and increasing the number of trials or step sizes would smooth the curves further. Therefore, the current setting is considered sufficient for comparative evaluation.

Although the oversampling rate $n/d$ is crucial in Phase Retrieval, it is not the only factor that matters. The length of the original signal $d$ also plays an important role in algorithm performance, in this case the success rate. To further explore this, we conducted additional trials on QRAF and its variants with varying dimensions, setting $d=30:30:300$ and $n/d=3:0.5:10$. For each parameter pair $(d,n/d)$, 30 trials were conducted to calculate the success rate, i.e. achieve $\text{dist}(x,z) < 10^{-5}$ within 1500 iterations. Notably, when $d = 100$, the slices of the following figures align with the curves in Fig. \ref{Pic_success_rate_B} respect to the corresponding algorithm.

\begin{figure*}[!t]
	\centering
	\subfloat[QRAF]{\includegraphics[width=2.0in]{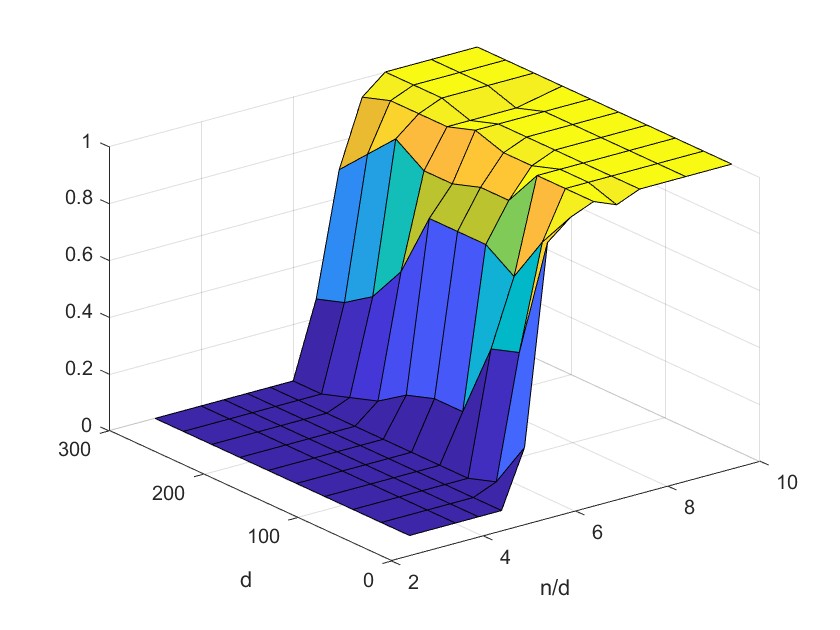}%
		\label{pic_QRAF}}
	\hfil
	\subfloat[QIRAF]{\includegraphics[width=2.0in]{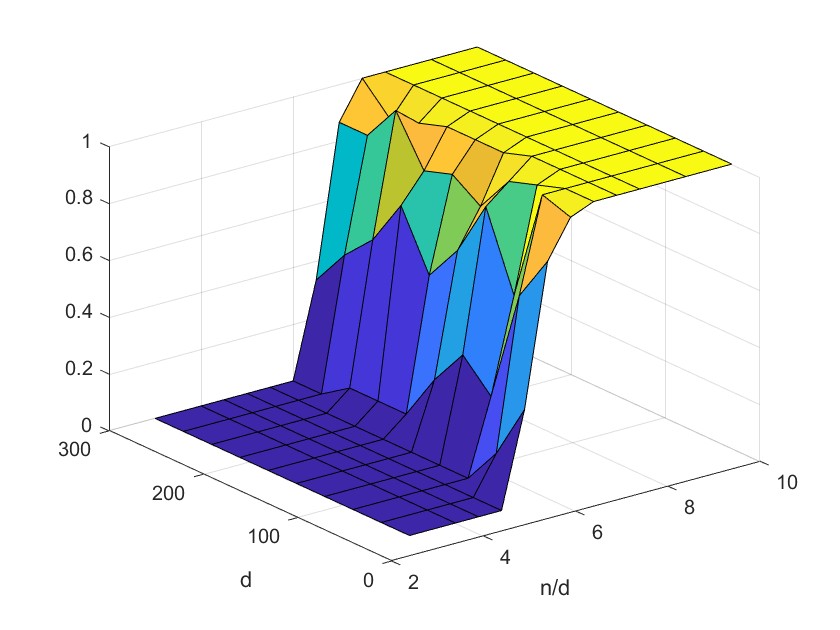}%
		\label{pic_QIRAF}}

	\subfloat[QARAF]{\includegraphics[width=2.0in]{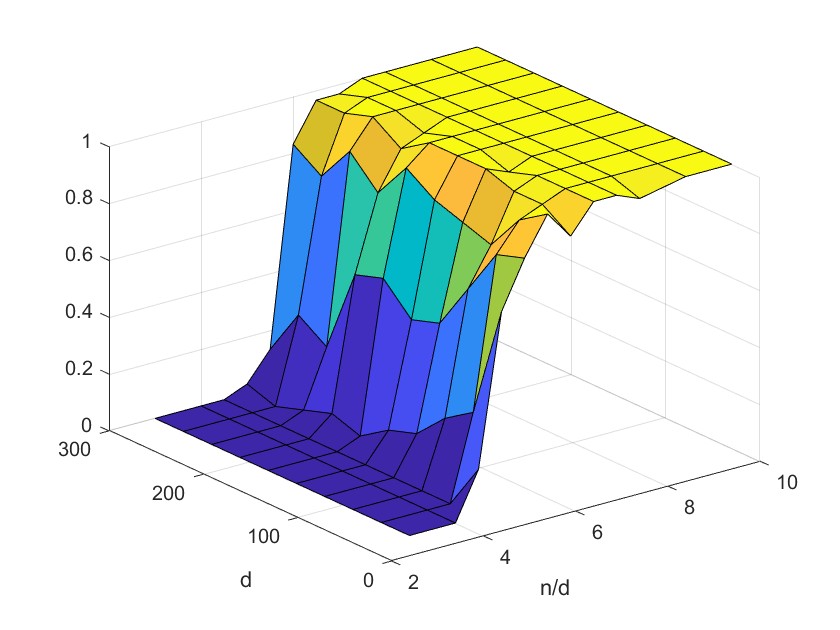}%
		\label{pic_QARAF}}
	\hfil
	\subfloat[QAdRAF]{\includegraphics[width=2.0in]{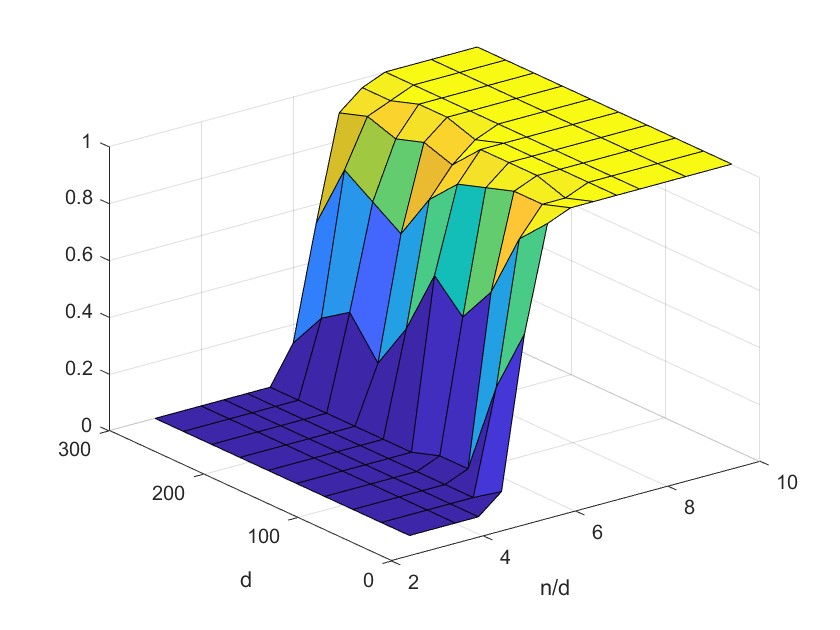}%
		\label{pic_QAdRAF}}
	\caption{Success rate of QRAF and its variants in various views}
	\label{fig_3D}
\end{figure*}

From Fig. \ref{fig_3D}, we can easily seen that the signal dimension also affect tested algorithms on convergence at some level. For instance, when $d=30$ and $n/d=6$, the success rate of QIRAF is $0.83$ while when $d=300$ the success rate is $0$ with the same $n/d$. Another example is when $d = 60$ and $n/d=6.5$ the success rate is $1$, but when $d=300$ and $n/d=6.5$, the success rate is $0.33$. (The original data is uploaded and can be verified.) 
In another word,  as $d$ increases, the ratio  $n/d$ must also be larger to achieve a high success rate. It is worth noting that, compared to  QRAF,  the QIRAF is more sensible to the size of vector $\bx\in \mathbb{H}^{d}$. And the QARAF is the most stable one, comparing to its counterparts.

{\bf iii. Compete test for quaternionic algorithms}\, Table \ref{t6a} provides an overview of the features of different algorithms. For these experiments, the parameters were set according to Table \ref{Table_parameters} and set $n/d = 9$, with $d = 64$ and $d = 100$ separately. Each algorithm was run for $100$ trials with the results averaged. It should be noted that for algorithms that success rate did not reach $1$, the average convergence time and iterations exclude the cases of non-converging.

\begin{table*}[htbp]
	\caption{Comparison of iteration count and time cost among quaternionic algorithms} \label{t6a}
	\centering
 \resizebox{\textwidth}{!}{
	\begin{tabular}{|l|c|c|c|c|c|c|c|c|c|c|}
		\hline
		Algorithm: & $d$ 				 & QWF			& QRWF		& QPAF		& QTAF 		& QRAF 			 &QIRAF      &QARAF      &QAdRAF\\
		\hline
		Success rate: &\multirow{3}*{64} & 0.81 		& 0.99 		&\textbf{1}	&\textbf{1} & \textbf{1} 	&\textbf{1}      &\textbf{1}     &\textbf{1}\\
		
		Iterations:&  					 &1280.12 		& 518.84 	&457.04		& 734.08	& 378.86 		& 382.50 			& \textbf{100.99} 	& 249.44 \\
		
		Time (s): &						 &93.52 		& 34.82 	&29.23		&52.54	 	& 27.45 	 & 11.32 			& \textbf{7.01} 	& 16.32\\
		\hline
		Success rate: &\multirow{3}*{100}&0.74 			& \textbf{1}&\textbf{1}	& 0.96 		& \textbf{1} & \textbf{1} 		& \textbf{1} 		& \textbf{1}\\
		
		Iterations:&  					 &1326.47  		& 594.07 	&501.27		& 824.76	& 409.63	 & 389.74 			& \textbf{106.85} 	& 272.92\\
		
		Time (s):&  					 &244.97 		&101.16		&86.10		&140.12 	& 68.02	 	& 19.11 			& \textbf{18.86} 	& 48.59\\
		\hline
	\end{tabular}}
	
\end{table*}

The performance of quaternionic Wirtinger Flow and quaternionic Amplitude Flow algorithms are exhibited in Table \ref{t6a}. Consider together with Fig. \ref{Pic_success_rate_A} and Fig. \ref{Pic_success_rate_B} when $d=100$, it is easily seen that QRAF and its variants surpasses other kinds algorithms in this table not only in success rate but also in convergence step and time consumption. This overall advantage is evident in both cases of $d=64$ and $d=100$.

When it comes to QRAF and its variants themselves, the competition becomes much intenser, but QARAF wins without doubts. All QRAF-type algorithms have a success rate of $1$ but QARAF has lower convergence step number and less time consumption by the acceleration technique. Note that QARAF is less influenced by the size of the signal, from $100.99$ to $106.85$ as $d$ increases from $64$ to $100$, comparing to other algorithms. This also aligns the conclusion of Fig. \ref{fig_3D}. 

Another outstanding algorithm is QIRAF, which must be studied carefully. Consider Table \ref{t6a} and Fig.\,\ref{Pic_Converg_compare} and Fig. \ref{Pic_success_rate_B} together, we can easily observe that the performance of QIRAF and QRAF are much similar. However, QIRAF enjoys much less computation than QRAF by the mini-batch technique without loosing much advantage of gradient. This is an essential virtue in practical.

Meanwhile, though QIRAF does not outperform QARAF when $d$ equals to $64$ or $100$ in both iteration count and time consumption, it can not be simply deemed as a complete fail to the latter. While QARAF gets the result of less time consumption by less convergence step, QIRAF achieved the same level of less time consumption by less computation. In another word, the average time per step of QIRAF is less than QARAF. Therefore, when $d$ gets larger (for instance when $d > 200$), QIRAF will use less time than QARAF in the same task while success rate still be $1$.

One more thing needs to be mentioned is that, QARAF having such an outstanding performance comes with a cost of algorithm complexity and CPU memory. Although the cost is not expensive enough to be considered essential, it can be clearly noticed in practical in time-consumption per iteration step. This is actually the secondary factor that QARAF and QIRAF have competition in time consumption, while QARAF is far beyond QIRAF in convergence step. Algorithm performance is always about the balance of time consumption, computation and machine memory.

\subsection{Experiment: Color Images} \label{cip2}
In this section, we compare different quaternionic algorithms in processing real images. The images used for testing are sourced from \textbf{The USC-SIPI Image Database} \footnote{https://sipi.usc.edu/database/}, i.e. 4.1.04 ($256 \times 256$), 4.1.05 ($512 \times 512$) and 4.2.03 ($256 \times 256$). Ideally, the entire image would be considered as the signal for reconstruction. However, due to computational limitations present not only in quaternionic algorithms but also in real and complex algorithms, we adopt the conventional approach of segmenting each image into smaller blocks and treating each block as a separate signal $\bx$ to be reconstructed. Specifically, in the first experiment, i.e. Fig. \ref{QWFsvsQAFs}, we test the picture 4.1.04 in the database as in \cite{chen_phase_2023}. The picture is $256 \times 256$, we split this picture into $32 \times 32$ pieces, then each piece is a $8 \times 8$ block containing $64$ pixel which is the signal $\bx$ and its size $d = 64$. In the second experiment, i.e. Fig. \ref{PQRAFs}, the figure 4.1.05 of the size $512 \times 512$ is also tested. The figure is divided also into  $32 \times 32$ pieces but with each block of $16 \times 16$ pixels. Therefore, in this experiment, the size of the signal $\bx$ is $d = 256$.

For these experiments, we set the $n/d$ ratio to $9$ and fixed the total number of iterations to $T = 300$ for each algorithm. The outcomes are presented in Fig. \ref{QWFsvsQAFs}. The results in Fig. \ref{QWFsvsQAFs} clearly highlight the performance of each algorithm. In comparison with the original image, only the reconstruction by PQRAF (see Fig. \ref{PQRAF_re}) contains no defective blocks, indicating the robustness of this algorithm. PQPAF and PQRWF also perform well, with only a few defective blocks, reflecting the varying strengths of WF and AF. Conversely, PQWF and PQTAF exhibit relatively weaker performance, illustrating that different WF and AF algorithms have trade-offs in terms of convergence speed and overall effectiveness.

\begin{figure*}[!t]
	\centering
	\subfloat[Original (PSNR,SSIM)]{\includegraphics[width=1.2in]{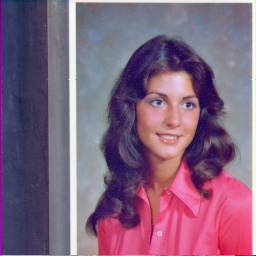}%
		\label{PQ_original}}\,
	\subfloat[PQWF (17,0.61)]{\includegraphics[width=1.2in]{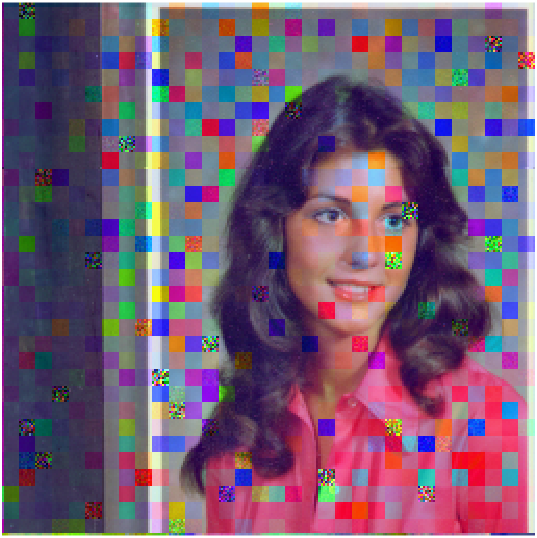}%
		\label{PQWF_re}}\,
	\subfloat[PQRWF (31,0.99)]{\includegraphics[width=1.2in]{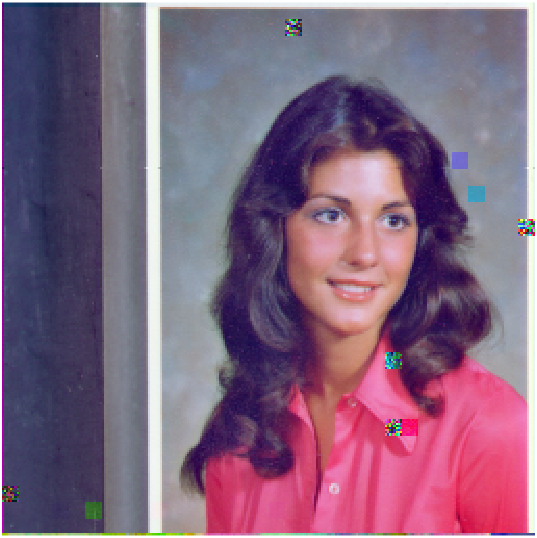}%
		\label{PQRWF_re}}
	
	\subfloat[PQPAF (45,1)]{\includegraphics[width=1.2in]{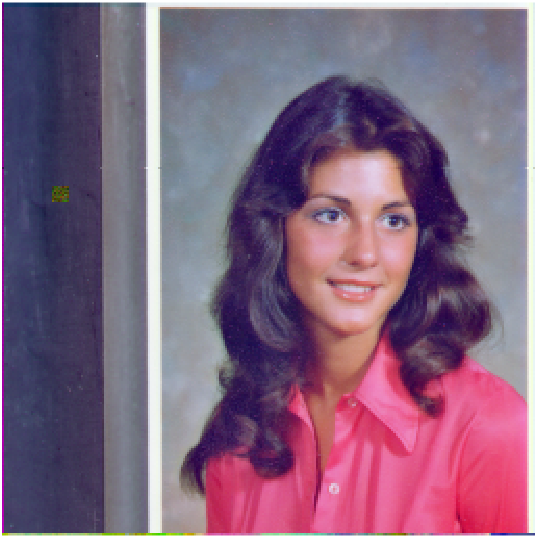}%
		\label{PQPAF_re}}\,
	\subfloat[PQTAF (25,0.92)]{\includegraphics[width=1.2in]{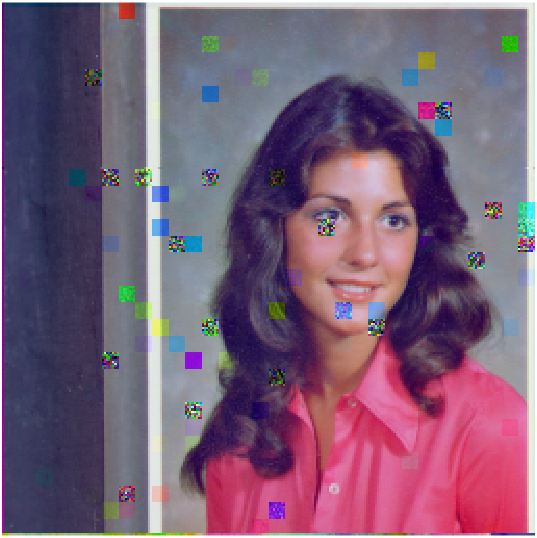}%
		\label{PQTAF_re}}\,
	\subfloat[PQRAF (96,1)]{\includegraphics[width=1.2in]{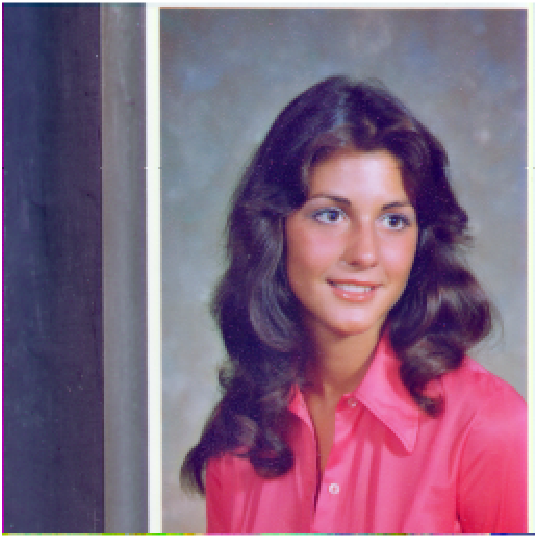}%
		\label{PQRAF_re}}
	\caption{Comparison of PQWFs and PQAFs}
	\label{QWFsvsQAFs}
\end{figure*}

To quantify these observations, we utilize two evaluation metrics: Peak Signal-to-Noise Ratio (PSNR) and Structural Similarity Index Measure (SSIM). PSNR measures the similarity between the original and reconstructed images, with higher values indicating closer similarity. SSIM assesses structural similarity, ranging from $0$ to $1$, where a score of $1$ denotes perfect similarity. The indices reveal that PQRAF significantly outperforms other algorithms, achieving the highest PSNR score of $96$ in this experiment. Additionally, both PQRAF and PQPAF attain the maximum SSIM score of $1$, leading the competition in this aspect. These indices corroborate the earlier qualitative observations with a more precise, quantitative evaluation.

In the subsequent experiment, we operate algorithms of QRAF and its variants, namely, PQARAF, PQAdRAF and PQIRAF to compare their performance (PQWF as control group). The setting of this experiment is aforementioned and the results are displayed in Fig. \ref{PQRAFs}. It is noteworthy that the QRAF-based algorithms share the same initialization method, therefore gradients are the only difference among these algorithms.

\begin{figure*}[!t]
	\centering
	\subfloat[Original (PSNR,SSIM)]{\includegraphics[width=1.2in]{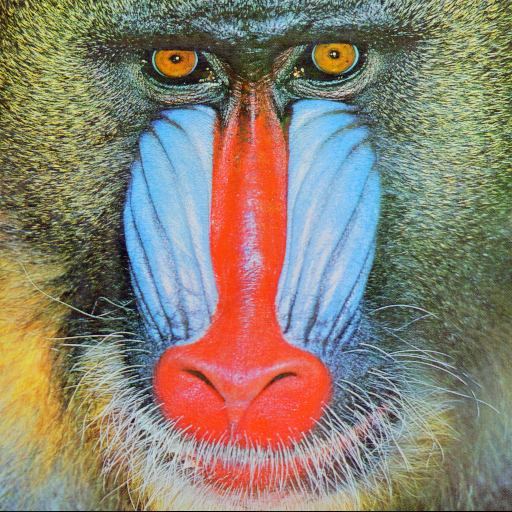}%
		\label{pic_trial_2_0}}\,
	\subfloat[PQRAF (40,1)]{\includegraphics[width=1.2in]{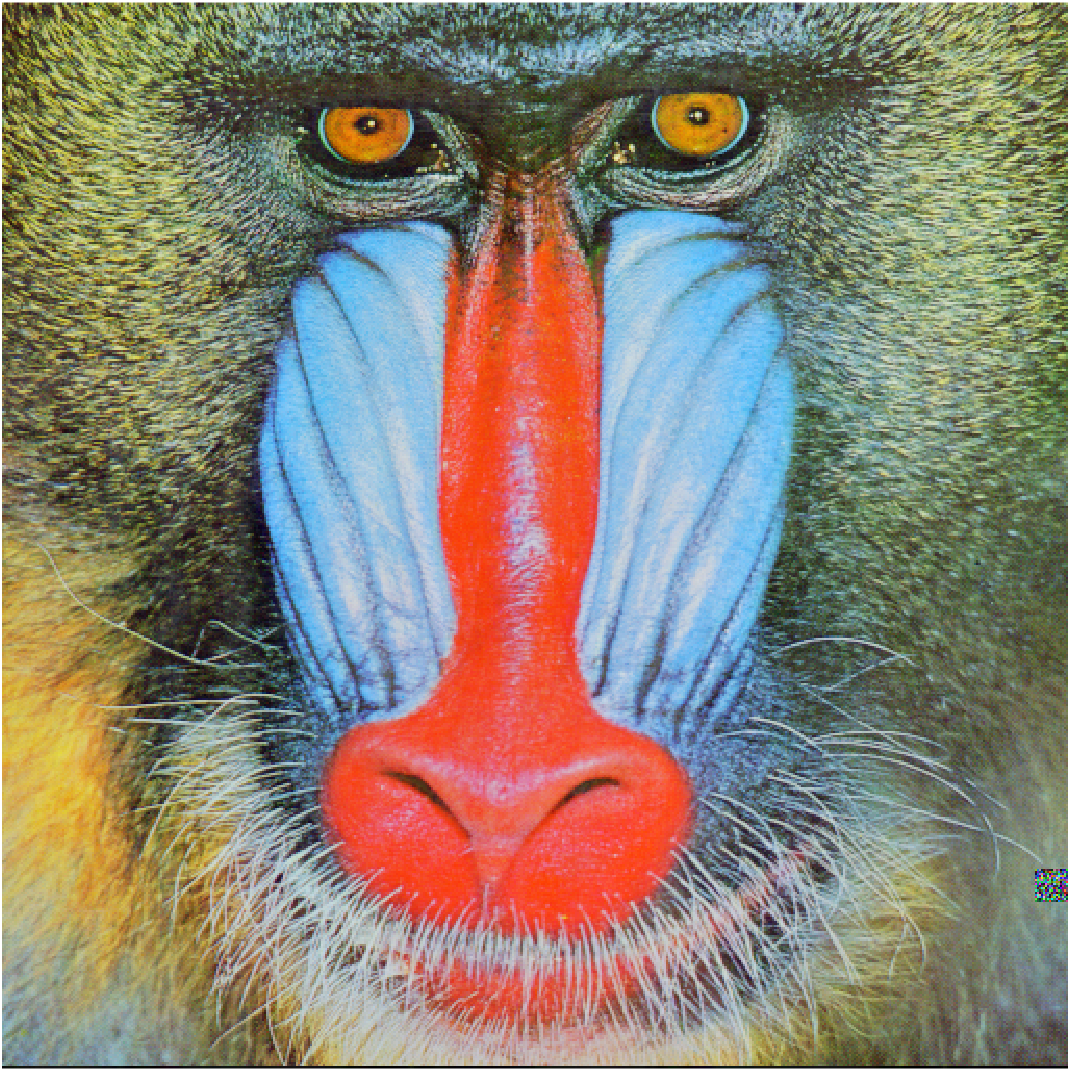}%
		\label{PQRAF_re2}}\,
	\subfloat[PQARAF (301,1)]{\includegraphics[width=1.2in]{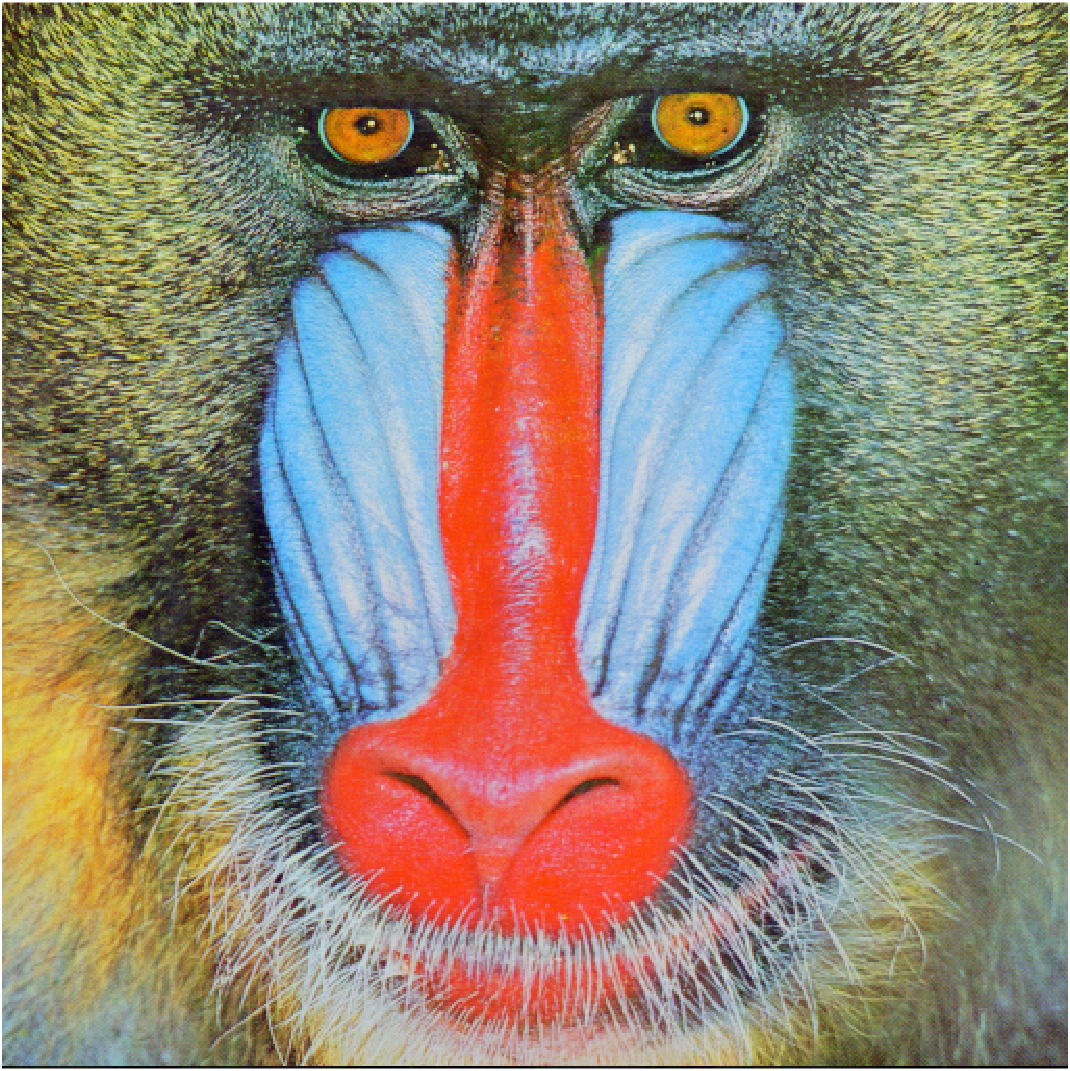}%
		\label{PQARAF_re}}
	
	\subfloat[PQAdRAF (133,1)]{\includegraphics[width=1.2in]{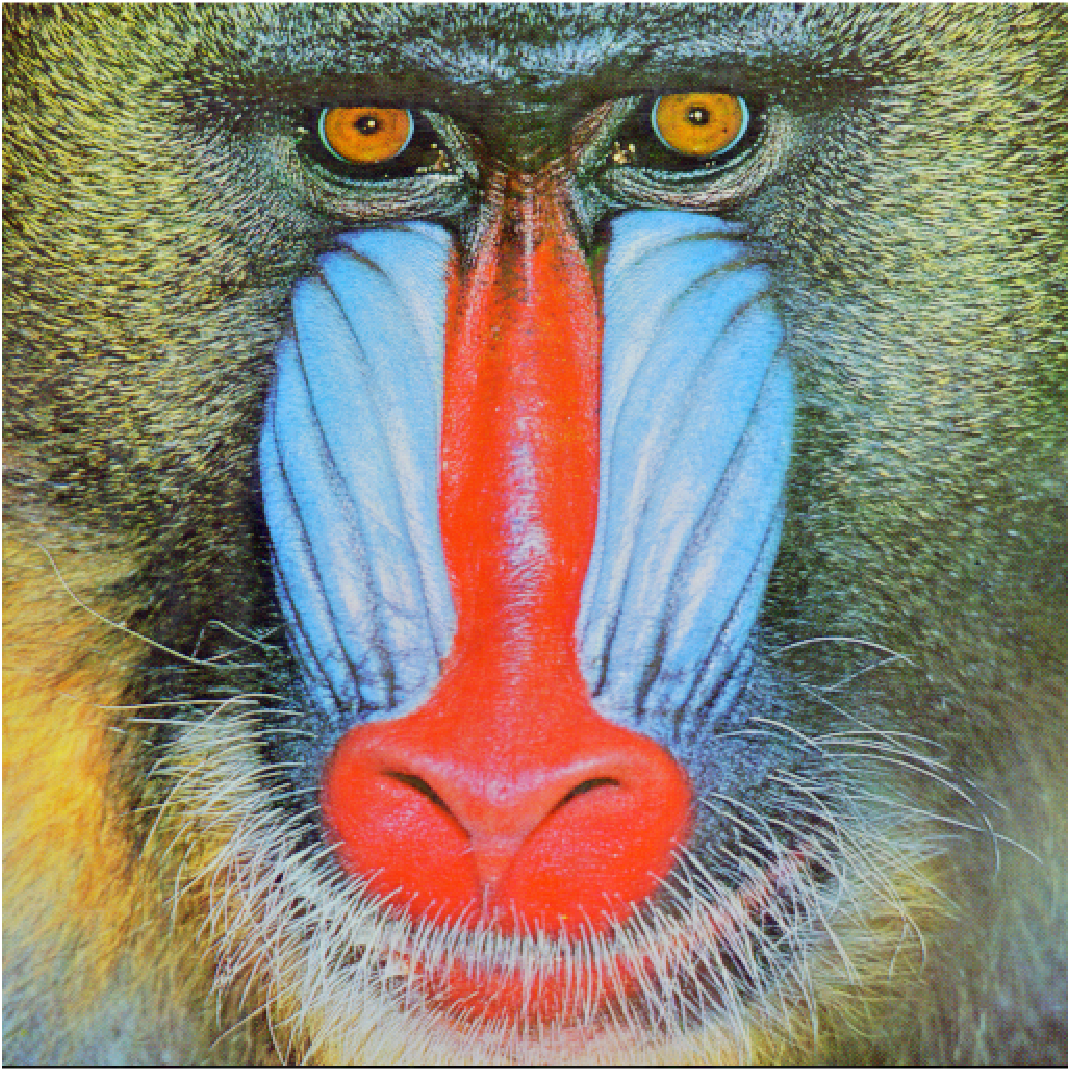}%
		\label{PQAdRAF_re}}\,
	\subfloat[PQIRAF (38,1)]{\includegraphics[width=1.2in]{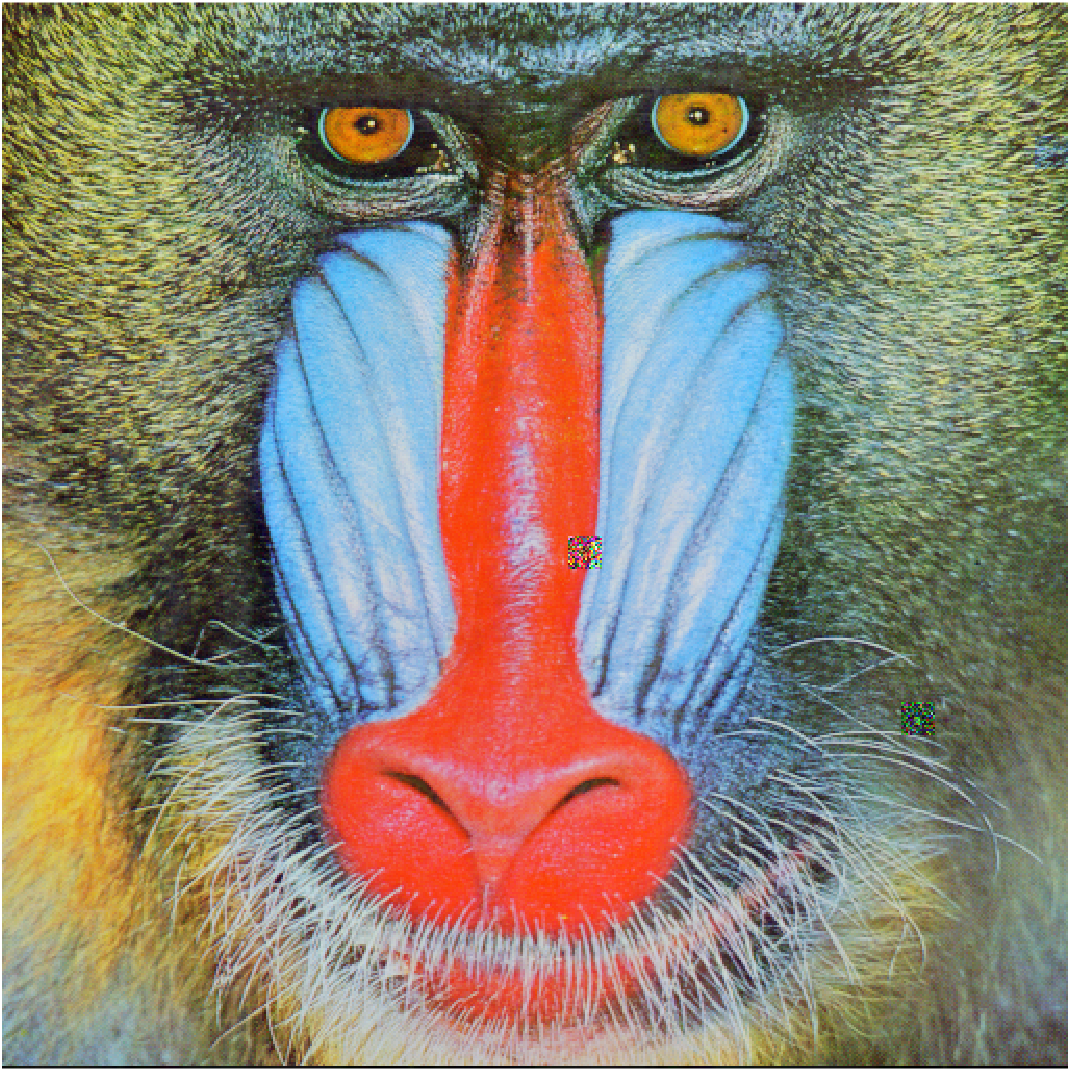}%
		\label{PQIRAF_re}}	\,
	\subfloat[PQWF (17,0.8)]{\includegraphics[width=1.2in]{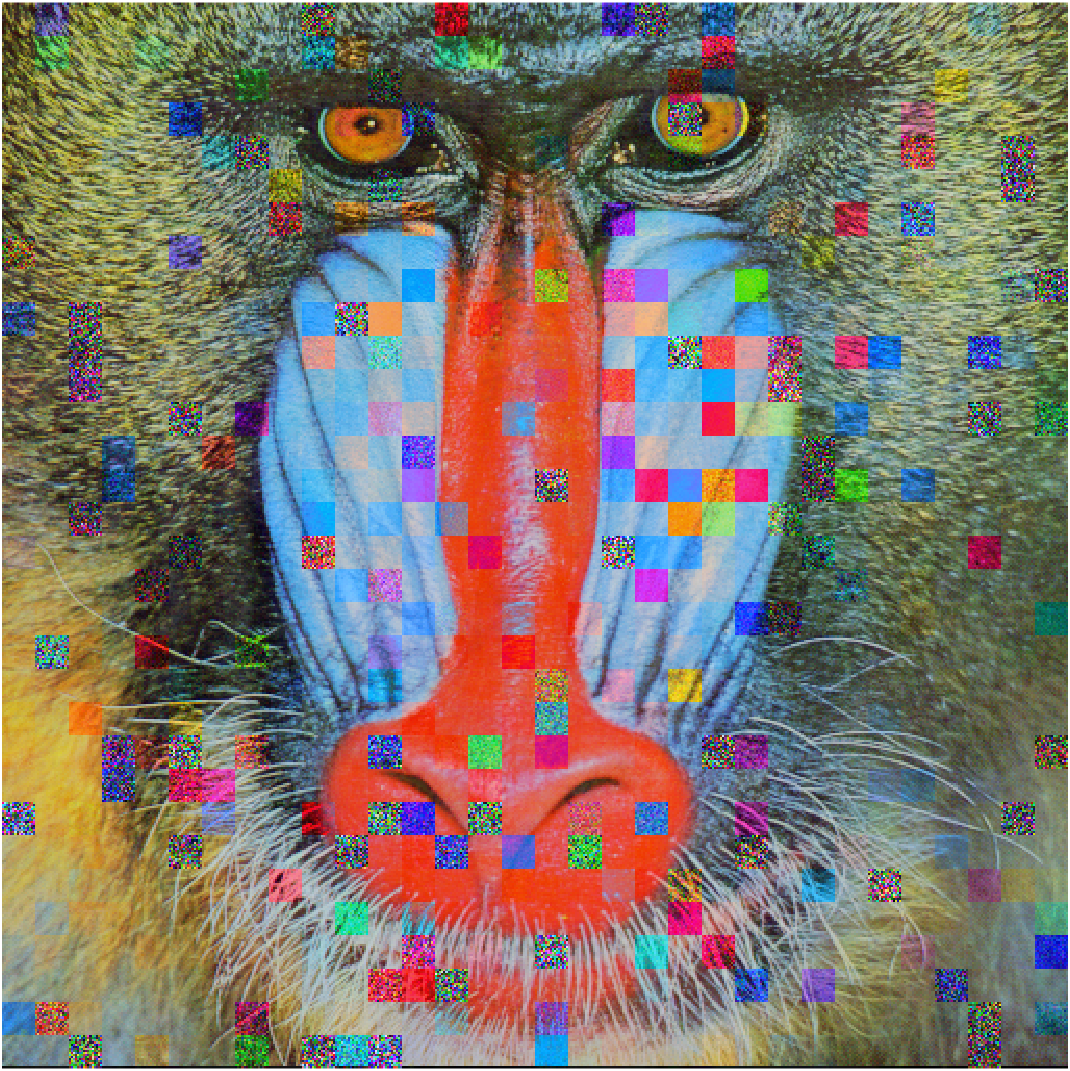}%
		\label{PQWF_re2}}
	\caption{Comparison of the variations of QRAF}
	\label{PQRAFs}
\end{figure*}

Compared to the previous experiment, QRAF and its variants all exhibit exceptional performance, successfully reconstructing the signals without any defective blocks. While visual inspection may not easily reveal differences among these results, the quantitative indices provide a more reliable assessment. In this experiment, SSIM loses its discriminatory power, as all algorithms achieve a score of $1$, reflecting their high accuracy. However, PQARAF stands out with a significantly higher PSNR of $301$, surpassing other variants. Meanwhile, the performance of PQRAF and PQIRAF in terms of PSNR remains very close, showing the same ability of reconstructing signals in practical.

\subsection{Comparison of conventional and quaternionic algorithms} \label{crq1}
It is {\em natural} and {\em frequently} to question the advantage of the Phase Retrieval in the quaternionic approach. While the problem can indeed be addressed in both the real and quaternionic domains, the quaternionic approach offers more than just an alternative method; it also provides certain practical benefits that enhance the research field. To illustrate this, we compare the Quaternionic Reweighted Amplitude Flow (QRAF) to the real Reweighted Amplitude Flow (RAF) algorithms.

Similar to \cite{chen_phase_2023}, we slightly modify the structure of the given quaternion signal to fit the real RAF algorithms. Here we propose two forms of real RAF algorithms, i.e., \textbf{RAF of a monochromatic model} (denoted by RAF-Mono) and \textbf{RAF of a concatenation model} (denoted by RAF-Conc).

For a pure quaternion signal $\bp \in \mathbb{H}^{d}_{p}$, the RAF of a monochromatic model separately reconstruct $\mathcal{P}^{i}(\bp)$, $\mathcal{P}^{j}(\bp)$ and $\mathcal{P}^{k}(\bp)$, which are three real signal in $\mR^{d}$. We denote by $\hat{\bz}^{i}$, $\hat{\bz}^{j}$ and $\hat{\bz}^{k}$ the reconstructed signal corresponding to each channel. To keep the metric consistent, the reconstruction error for real RAF of a quaternion signal $\bp$ is measured as $\left(\sum_{h=i,j,k}{\rm dist}_{p}\left(\mathcal{P}^{h}(\bp), \hat{\bz}^{h} \right)^{2} \right)^{1/2} $. On the other hand, the concatenation model of real RAF for a pure quaternion signal $\bp \in \mathbb{H}^{d}_{p}$ also split the signal as $\mathcal{P}^{i}(\bp)$, $\mathcal{P}^{j}(\bp)$ and $\mathcal{P}^{k}(\bp)$, and concatenate the three real vector into a new signal, i.e. $\hat{\bp} = \left[\left(\mathcal{P}^{i}(\bp)\right)^{T}, \left(\mathcal{P}^{i}(\bp)\right)^{T}, \left(\mathcal{P}^{i}(\bp)\right)^{T}\right]^{T} \in \mR^{3d}$. We denote by $ \hat{\bz} \in \mR^{3d}$ the reconstructed signal, then the measurement for reconstruction error of this algorithm is defined as ${\rm dist}_{p}\left(\hat{\bp}, \hat{\bz} \right)$.

{\bf i. Compete test in convergence and success rate}\, Firstly, we investigate the performance of algorithms in convergence and success rate as before. To compare with RAF-Mono and RAF-Conc, we test pure quaternionic algorithms, since the given signal here is pure quaternion $\bp \in \mathbb{H}^{d}_{p}$.

\begin{figure}[htbp]
	\centering
	\begin{minipage}{0.45\textwidth}
		\centering
		\includegraphics[width=1.0\textwidth]{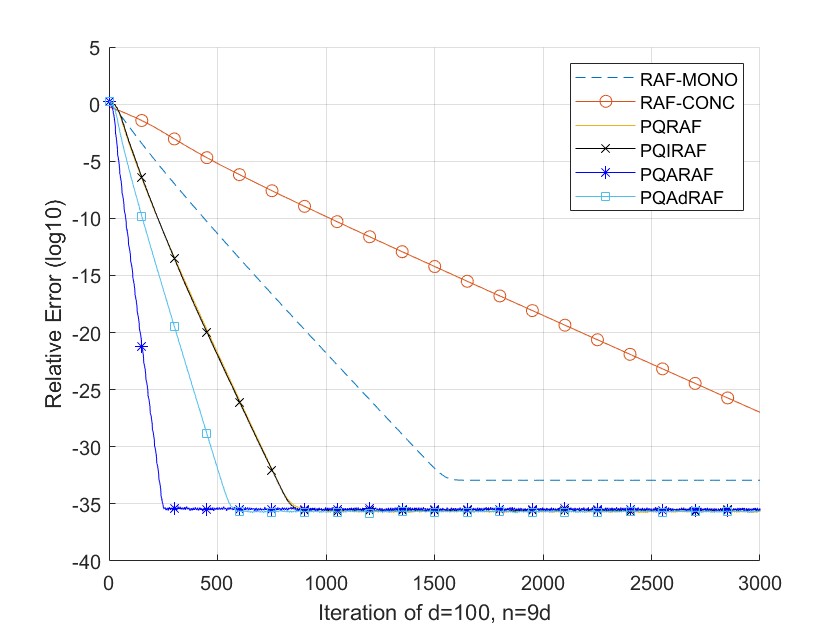}
		\caption{Convergence of real and pure quaternion algorithms}
		\label{Pic_Conv_Q_R}
	\end{minipage}
	\hfill
	\begin{minipage}{0.45\textwidth}
		\centering
		\includegraphics[width=1.0\textwidth]{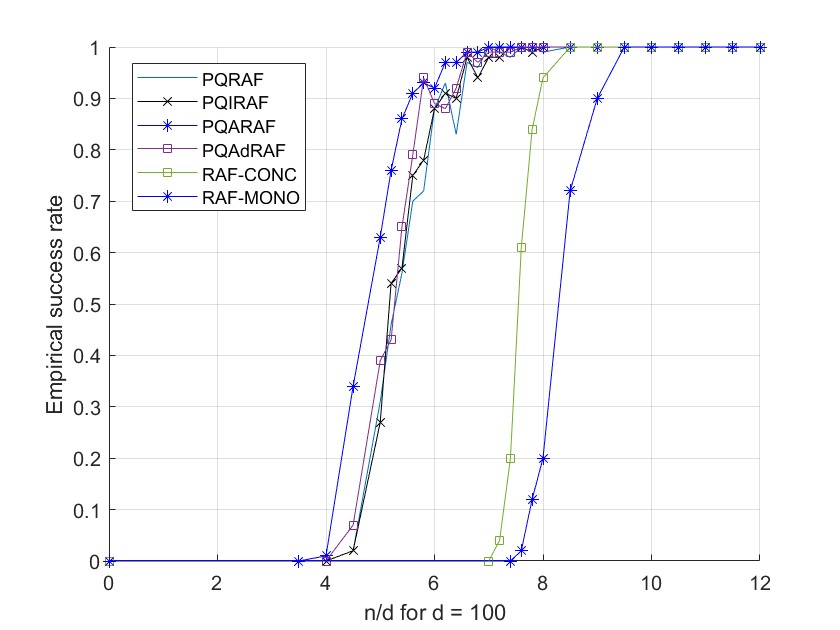}
		\caption{Success rate of real and pure quaternion algorithms}
		\label{Pic_SR_Q_R}
	\end{minipage}
\end{figure}

It can be easily seen from Fig. \ref{Pic_Conv_Q_R} that quaternionic algorithms converge faster than real algorithms, which PQARAF reach the relative error of $\ln(10^{-35})$ within $500$ iterations. From Fig. \ref{Pic_SR_Q_R}, we can observe that PQRAF and its variants have much lower $n/d$ rate to achieve the success rate of $1$ in this experiment. 

\begin{table*}[htbp]
	\caption{Comparison of iteration count and time cost among algorithms}
	\centering
  \resizebox{\textwidth}{!}{
	\begin{tabular}{|l|c|c|c|c|c|c|c|}
		\hline
		Algorithm: & $d$ 					& PQRAF 		& PQIRAF 		&PQARAF 		& PQAdRAF 	& RAF-MONO 	& RAF-CONC \\
		\hline
		Success rate: &\multirow{3}*{64} 	&\textbf{1}		&\textbf{1} 	&\textbf{1}		& \textbf{1}&\textbf{1}	&\textbf{1}\\
		
		Iterations:&  						&281.12			&278.74			&\textbf{84.95}	&192.08 	& 487.78 	& 1149.13 \\
		
		Time (s): & 						&20.65			&8.12 			&6.37 	&13.58 	&\textbf{2.69}&3.21 \\
		\hline
		Success rate: &\multirow{3}*{100} 	&\textbf{1} 	& \textbf{1}	&\textbf{1} 	&\textbf{1} &\textbf{1} &\textbf{1}\\
		
		Iterations:&  						&299.54			&288.64			&\textbf{93.52}	&211.16		&484.43 &1160.02 \\
		
		Time (s): & 						&52.89			&15.22			&16.47			&37.30		&\textbf{4.39}	& 4.96 \\
		\hline
	\end{tabular}}
	
	\label{Table_RAF_realVSqall}
\end{table*}

However, one advantage of real algorithms that quaternionic ones cannot have is the time consumption. From Table \ref{Table_RAF_realVSqall}, we can easily observe that, though PQRAF and its variants have much less average convergence steps, they cannot compete with RAF-Mono and RAF-Conc in time consumption. 

The only reason of this situation is that we use MATLAB to run the experiments, which is a professional software to deal with matrix computation especially in real and complex field, thus for RAF-Mono and RAF-Conc, it is suitable for these algorithm. There are no such a special software for quaternion algorithms at present. Much of the time consumption is used on quaternion-to-real and quaternion-to-complex matrix transformation to complete the computation. 

{\bf ii. Comparison of quaternionic and real algorithms in color image}\, In this section, we compare real PR algorithms and the quaternionic ones in reconstructing color images. The image to be tested is 4.1.05 from the data set. The size of the image is $256 \times 256$. Similar to the first experiment in Fig. \ref{QWFsvsQAFs}, we split the image into $32 \times 32$ piece with each piece is of $8 \times 8$ pixels. The total iteration is still $300$ and we set $n/d$ rate to $9$ as always. The results is shown in Fig. \ref{QRAFvsRAF}.

\begin{figure*}[!t]
	\centering
	\subfloat[PQRAF (87,1)]{\includegraphics[width=1.2in]{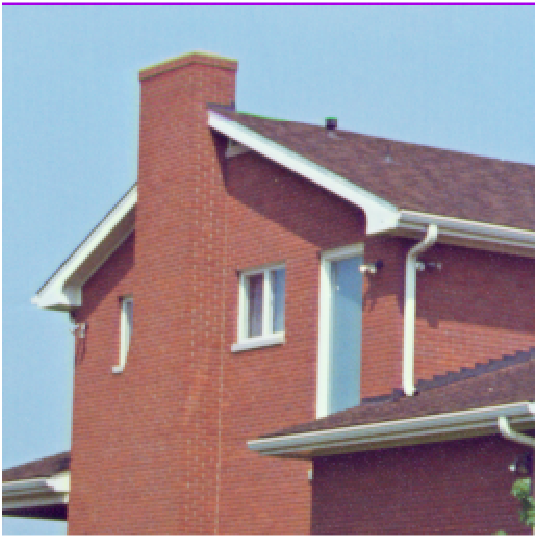}%
		\label{pic_trial_3_PQRAF}}\,
	\subfloat[PQARAF (282,1)]{\includegraphics[width=1.2in]{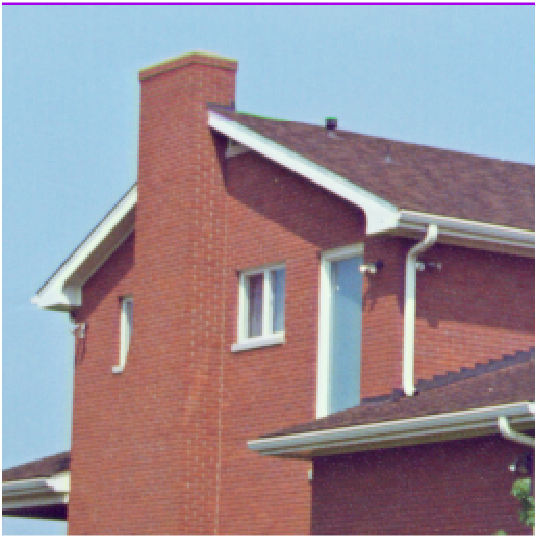}%
		\label{pic_trial_3_PQARAF}}\,
	\subfloat[PQAdRAF (125,1)]{\includegraphics[width=1.2in]{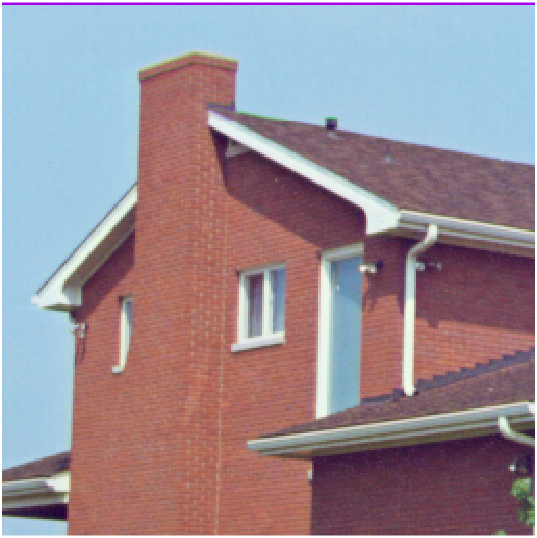}%
		\label{pic_trial_3_PQAdRAF}}
	
	\subfloat[PQIRAF (64,1)]{\includegraphics[width=1.2in]{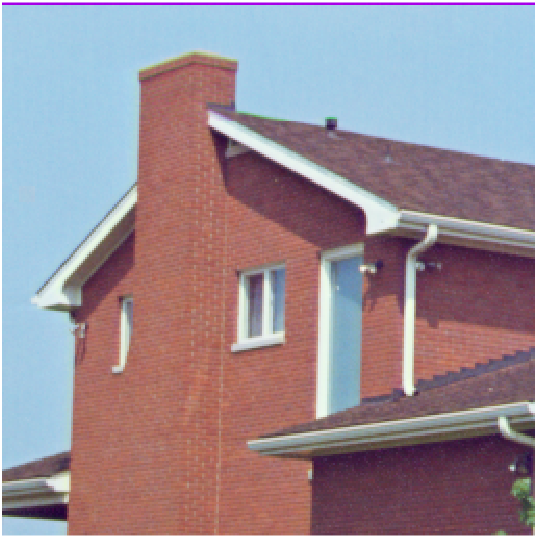}%
		\label{pic_trial_3_PQIRAFF}}\,
	\subfloat[RAF-MONO (27,0.95)]{\includegraphics[width=1.2in]{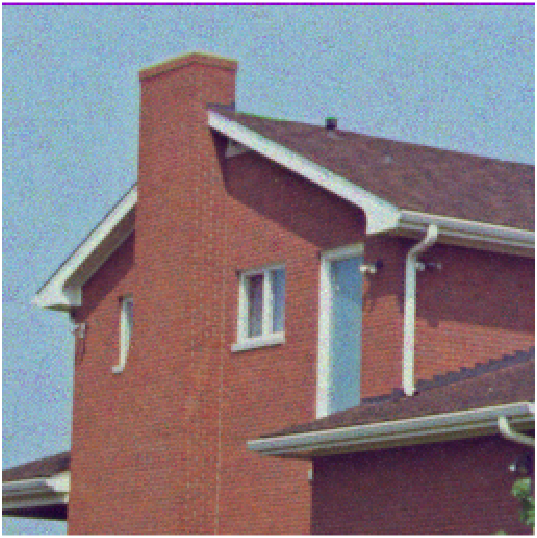}%
		\label{pic_trial_3_MONO}}\,
	\subfloat[RAF-CONC (28,0.97)]{\includegraphics[width=1.2in]{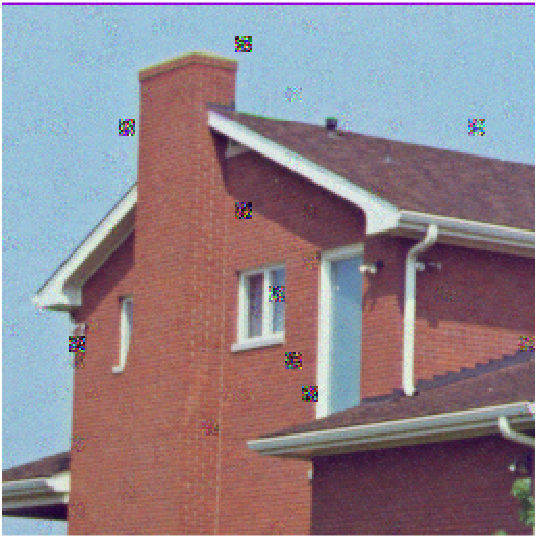}%
		\label{pic_trial_3_CONC}}
	\caption{Comparison of quaternionic and real algorithms}
	\label{QRAFvsRAF}
\end{figure*}

From the result, it is evident that PQRAF and its variants achieve superior results with the same number of iterations compared to real RAF algorithms. After the same number of iterations, RAF-Conc still displays some defective blocks, and RAF-Mono appears generally darker than the original image. In contrast, all quaternionic algorithms produce well-reconstructed images without any noticeable defects.

These differences can be quantitatively assessed using indices similar to those employed in the experiment depicted in Fig.\,\ref{QWFsvsQAFs}. Specifically, PQARAF significantly outperforms other algorithms, achieving a PSNR of $282$ and an SSIM of $1$. In comparison, the real algorithms RAF-Mono and RAF-Conc attain PSNR values of $27$ and $28$, and SSIM scores of $0.95$ and $0.97$, respectively. The relatively lower PSNR and non-optimal SSIM values for the real algorithms indicate their inferior performance relative to the quaternionic algorithms. 

However, one thing that cannot be neglected is that real RAF algorithms are evidently faster than quaternionic ones in reconstructing an image as we already mentioned. By simply increasing the $n/d$ rate without any time consumption piles up, real algorithms can surpass the quaternionic ones in convergence as well. However, such comparisons would no longer adhere to the same standard, and thus fall outside the scope of this study.

\section{{Conclusions}}\label{p8}

In this work, we addressed the quaternionic phase retrieval (QPR) problem by proposing the Quaternionic Reweighted Amplitude Flow (QRAF) algorithm, a novel generalization of the real/complex amplitude flow framework to the quaternion domain. Our method innovatively extends the reweighted gradient descent paradigm by carefully accounting for the noncommutative nature of quaternion multiplication. In addition, we developed three algorithmic variants of QRAF to further enhance convergence and stability. Comprehensive experiments on both synthetic datasets and real-world color images demonstrate that QRAF consistently outperforms existing quaternionic methods such as QWF, QTWF, and QTAF in terms of convergence speed, recovery accuracy, and robustness. In particular, our methods achieve higher success rates under limited measurements, and demonstrate superior PSNR and SSIM scores in practical image reconstruction tasks.

To further improve theoretical rigor, we introduced the Quaternionic Perturbed Amplitude Flow (QPAF) algorithm, which is based on a perturbed amplitude model and inherits provable linear convergence properties from its complex counterpart. While QPAF achieves slightly lower empirical performance than QRAF, its simplicity and theoretical guarantees make it a valuable complementary approach.


\section{{Future Work}} \label{p10}

While the proposed methods achieve strong empirical results, several  questions remain  open for future investigation.

\textbf{(1) Theoretical analysis of QRAF:} One notable challenge is the absence of a proof for the Local Convergence Radius (LCR) condition in the QRAF, as outlined in Problem \ref{prop1}. This limitation is not unique to the quaternionic setting but is also observed in the complex setting. We anticipate that future work establishing the LCR condition for complex signals will facilitate its extension to the quaternionic framework without significant complexity. 

\textbf{(2) Efficient quaternionic computation:} The computational cost of QPR algorithms remains high compared to their real-valued analogues. This is largely due to the lack of hardware-accelerated support and optimization frameworks for quaternion algebra. Future work may explore dedicated quaternionic linear algebra libraries, GPU-accelerated quaternion solvers, or fast approximation schemes that retain algebraic structure while reducing computational complexity.

\textbf{(3) Parameter tuning strategies:} The performance of QRAF and QPAF is sensitive to initialization and step-size parameters, which are less well understood in quaternionic optimization. Developing adaptive parameter selection schemes—such as learning-based hyperparameter tuning or curvature-aware step size adjustment—could improve robustness across diverse signal types. It is also an interesting topic.




\section*{Acknowledgments}
{We thank the reviewers for their insightful comments, which  significantly improved the clarity of the paper.} This work  was partially supported  by NSFC Grant No.12101451.






\section{Appendix}
To provide additional insight into the behavior of  QRAF-based algorithms, we present supplementary experiments focusing on key implementation parameters. In particular, we investigate the influence of the batch size in QIRAF, a mini-batch variant of QRAF designed for improved computational efficiency.

\subsection{Effect of Batch Size in QIRAF}

QIRAF is designed to accelerate the convergence of QRAF by using mini-batch gradients, thereby reducing computational cost while maintaining comparable recovery performance. A crucial factor in QIRAF is the choice of batch size $\ABS{\Gamma}$. Intuitively, a smaller batch size reduces computation per iteration but may lead to less accurate updates and degraded convergence. Conversely, a larger batch size approaches full gradient descent, improving stability but sacrificing speed. Thus, selecting an appropriate batch size involves a trade-off between efficiency and accuracy.

To evaluate this trade-off, we fix the signal dimension $d = 100$ and oversampling ratio $n/d = 10$, and examine four batch sizes: $\ABS{\Gamma} = 0.1n$, $0.15n$, $0.2n$, and $0.25n$. These are compared with the default setting used in our main experiments, where $\ABS{\Gamma} = 2^k$. Fig. \ref{Pic_bs_cv_all} illustrates the convergence behavior across different batch sizes, with a zoomed-in view shown in Fig. \ref{Pic_bs_cv_local}. The results show that all batch sizes achieve similar overall convergence trends, but subtle differences emerge upon closer inspection. Specifically, larger batches yield smoother and slightly faster convergence, as expected.
\begin{figure}[htbp]
	\centering
	\begin{minipage}{0.45\textwidth}
		\centering
		\includegraphics[width=1.0\textwidth]{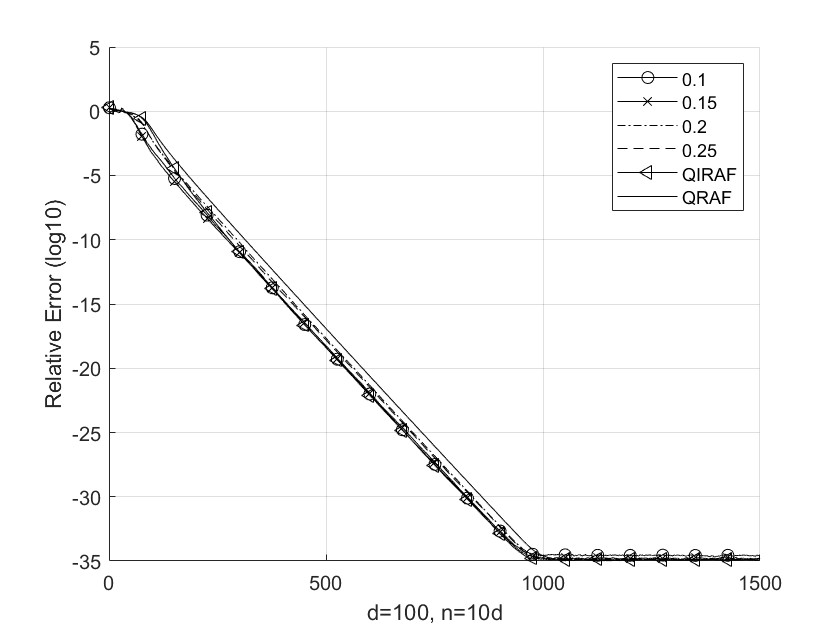}
		\caption{Convergence of QIRAF with different batch sizes}
		\label{Pic_bs_cv_all}
	\end{minipage}
	\hfill
	\begin{minipage}{0.45\textwidth}
		\centering
		\includegraphics[width=1.0\textwidth]{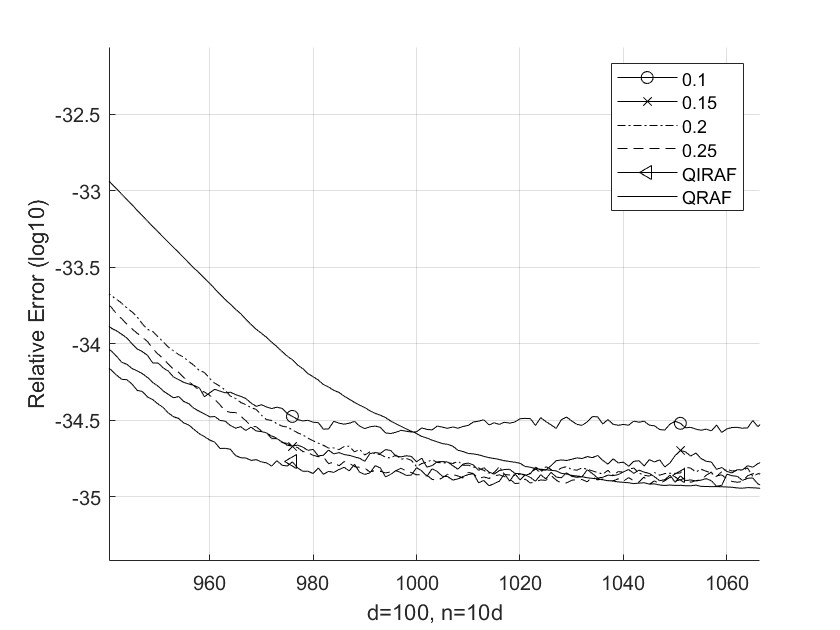}
		\caption{Zoomed-in view of convergence (from Fig.~\ref{Pic_bs_cv_all})}
		\label{Pic_bs_cv_local}
	\end{minipage}
\end{figure}

To further assess the practical impact of batch size, we conduct $100$ trials for each setting and evaluate two metrics: success rate and average runtime. The results are summarized in Fig. \ref{Pic_bs_sr} and \ref{Pic_bs_tc}. As shown in Fig.~\ref{Pic_bs_sr}, the success rate of QIRAF improves as the batch size increases, approaching that of QRAF when $\ABS{\Gamma}$ becomes large. This is consistent with the fact that QIRAF reduces to QRAF when the batch size equals the full sample size ($\ABS{\Gamma} = n$). However, Fig.~\ref{Pic_bs_tc} reveals that this gain in accuracy comes at the cost of increased computational time.
\begin{figure}[htbp]
	\centering
	\begin{minipage}{0.45\textwidth}
		\centering
		\includegraphics[width=1.0\textwidth]{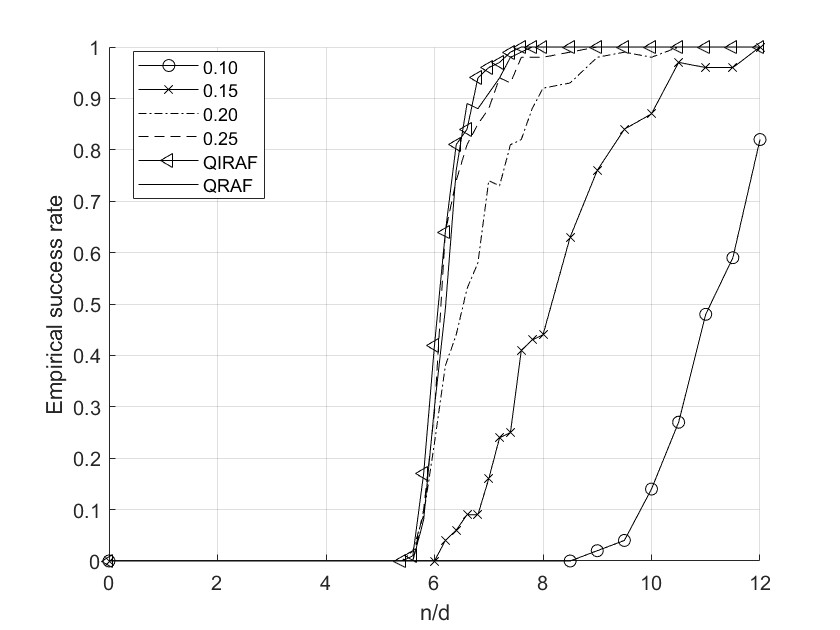}
		\caption{Success rate vs. batch size}
		\label{Pic_bs_sr}
	\end{minipage}
	\hfill
	\begin{minipage}{0.45\textwidth}
		\centering
		\includegraphics[width=1.0\textwidth]{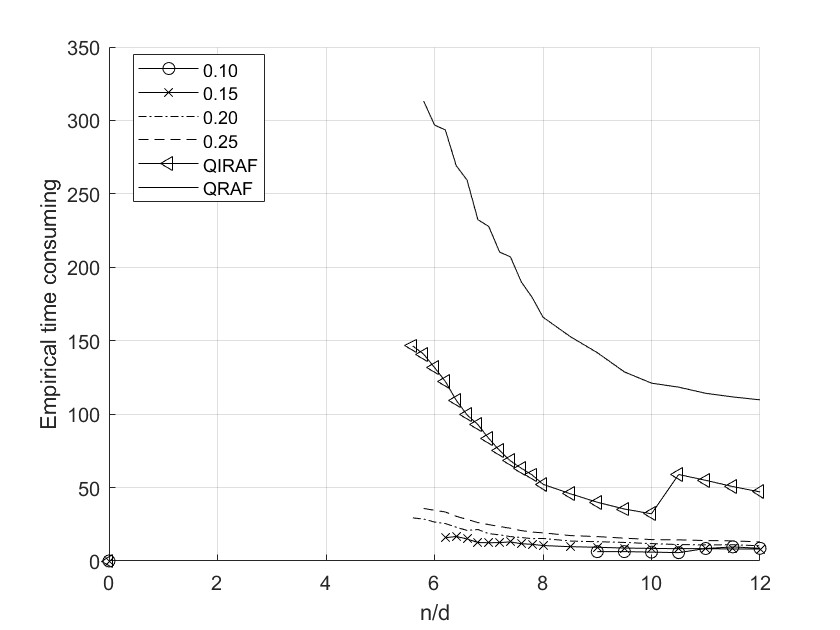}
		\caption{Average time consumption vs. batch size}
		\label{Pic_bs_tc}
	\end{minipage}
\end{figure}

In summary, the batch size $\ABS{\Gamma}$ plays a critical role in balancing efficiency and accuracy in QIRAF. The proposed setting of $\ABS{\Gamma} = 2^k$ offers a practical compromise, achieving competitive success rates with significantly reduced runtime. These findings support the robustness and practicality of QIRAF under different batch configurations.

\subsection{Phase comparison of QRAFs}
To gain further insight into the global convergence behavior of QRAF-based algorithms, we conduct additional experiments of algorithms on images. We evaluated the PSNR and SSIM of each algorithm over 1500 iterations. The results are presented in Fig. \ref{Pic_conv_psnr} and \ref{Pic_conv_ssim}.

\begin{figure}[!htbp]
	\centering
	\begin{minipage}{0.45\textwidth}
		\centering
		\includegraphics[width=1.0\textwidth]{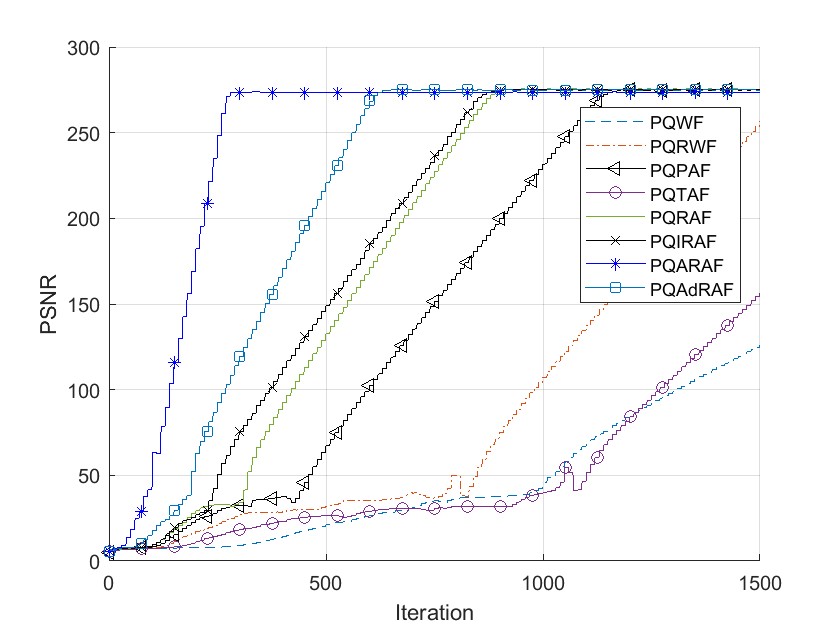}
		\caption{Convergence of PSNR}
		\label{Pic_conv_psnr}
	\end{minipage}
	\hfill
	\begin{minipage}{0.45\textwidth}
		\centering
		\includegraphics[width=1.0\textwidth]{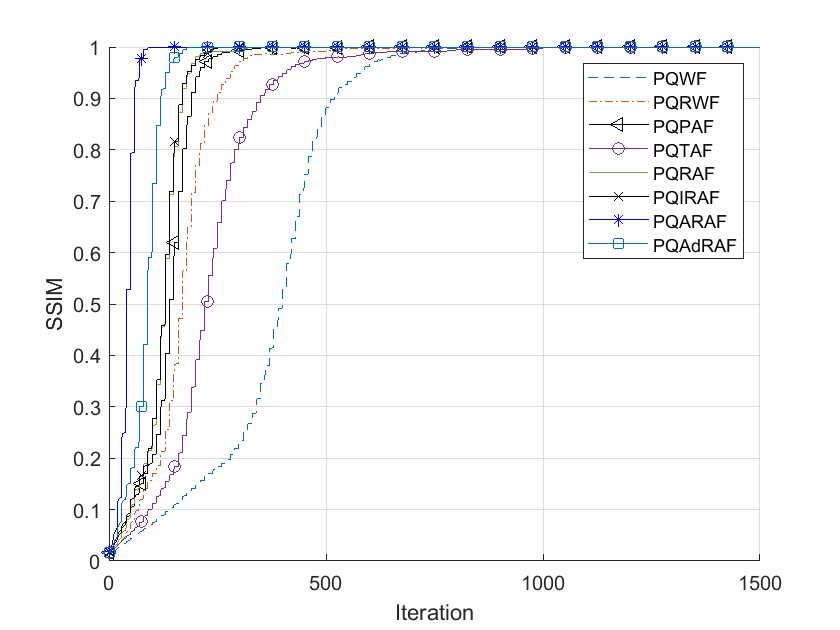}
		\caption{Convergence of SSIM}
		\label{Pic_conv_ssim}
	\end{minipage}
\end{figure}

As shown in Figure \ref{Pic_conv_psnr}, PSNR gradually increases and converges after a certain number of iterations, which is consistent with the convergence patterns observed in our synthetic experiments. This agreement further validates the reliability of our proposed methods.

Figure \ref{Pic_conv_ssim} illustrates SSIM values over the same iterations. Due to its bounded range in $[0,1]$, SSIM reveals more clearly the early-stage convergence of the algorithms. Most methods reach near-perfect SSIM relatively quickly, demonstrating that SSIM is less sensitive in distinguishing final performance at high fidelity levels. Nevertheless, the results support the superior performance of our algorithms.

To provide intuitive insights into the visual quality evolution, we present selected reconstructed images at key iterations ($T = 0$, $200$, $300$, $800$, and $1500$) in Figure \ref{pic_gen_all}. These samples reveal rapid improvement in image clarity and detail, with differences becoming visually negligible after $T = 300$. 

\begin{figure}[htbp]
	\centering
	\captionsetup[subfloat]{labelsep=none,format=plain,labelformat=empty}
	
	\subfloat[PQRAF:T=0]{\includegraphics[width=0.7in]{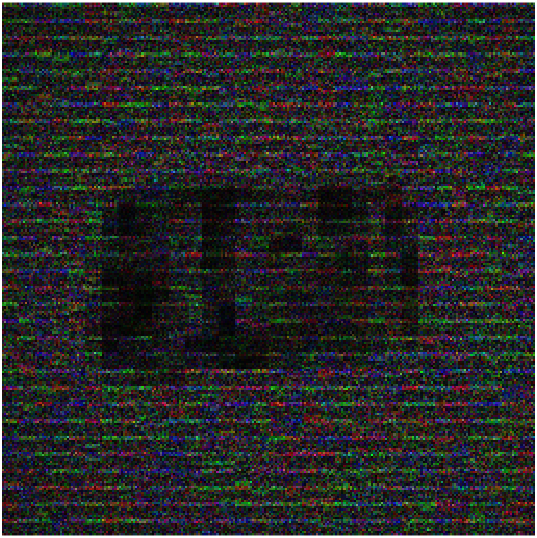}%
		}\,
	\subfloat[T=200(29,0.98)]{\includegraphics[width=0.7in]{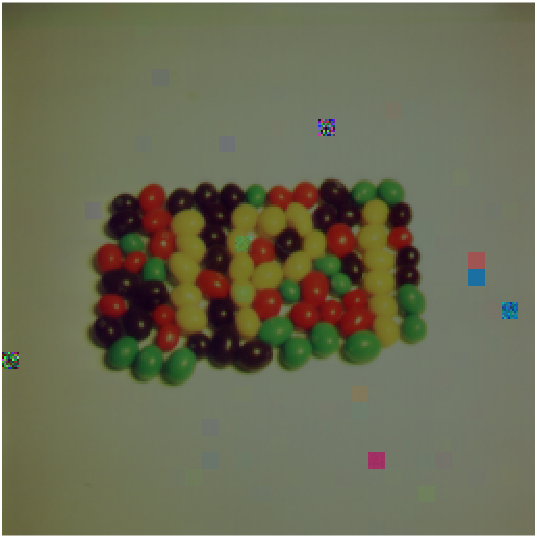}%
		}\,
	\subfloat[T=300(34,0.99)]{\includegraphics[width=0.7in]{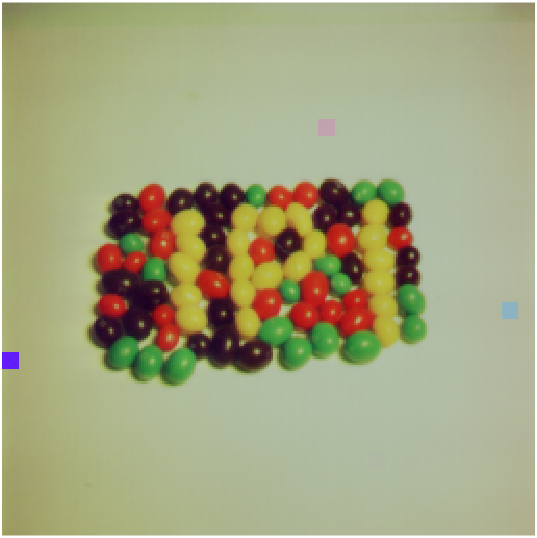}%
		}\,
	\subfloat[T=800(246,1)]{\includegraphics[width=0.7in]{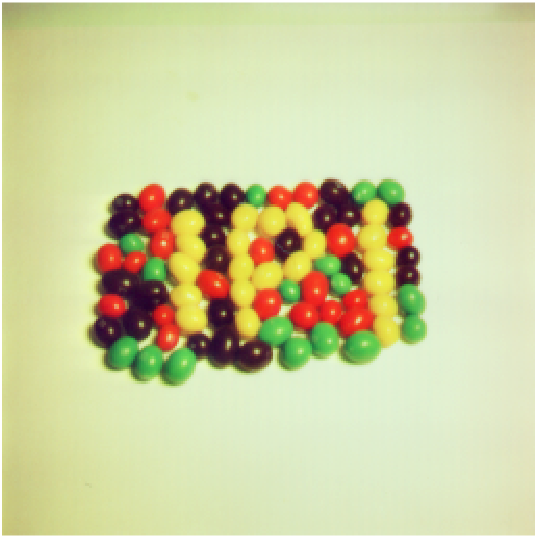}%
		}\,
	\subfloat[T=1500(275,1)]{\includegraphics[width=0.7in]{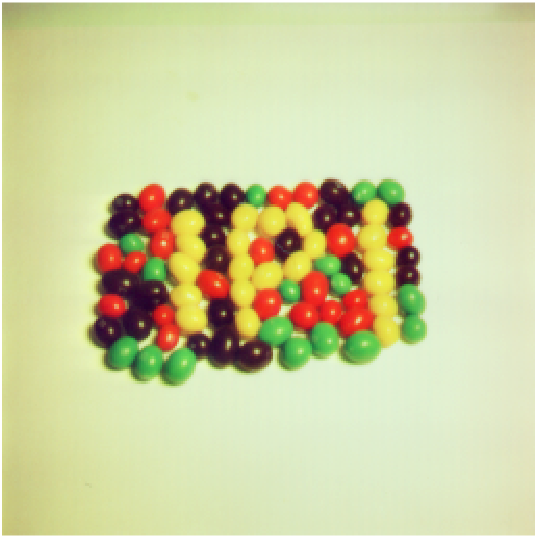}%
		}

    \subfloat[PQIRAF:T=0]{\includegraphics[width=0.7in]{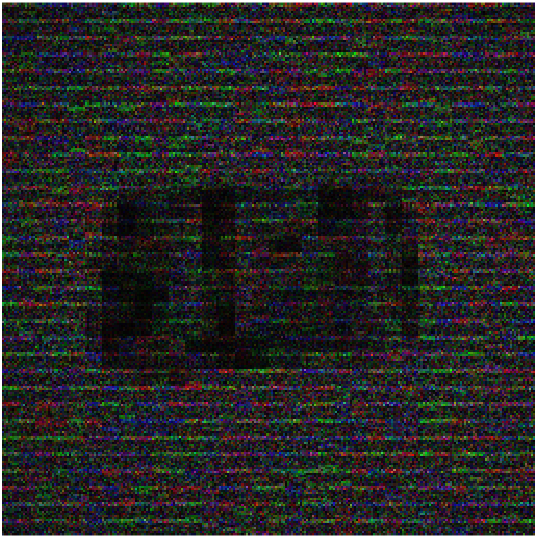}%
		}\,
	\subfloat[T=200(27,0.98)]{\includegraphics[width=0.7in]{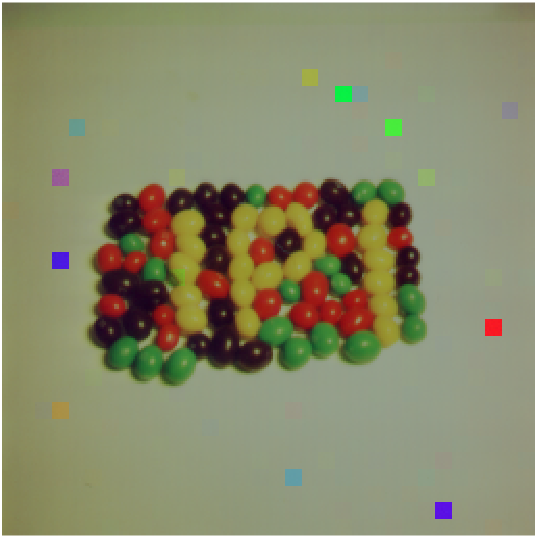}%
		}\,
	\subfloat[T=300(76,1)]{\includegraphics[width=0.7in]{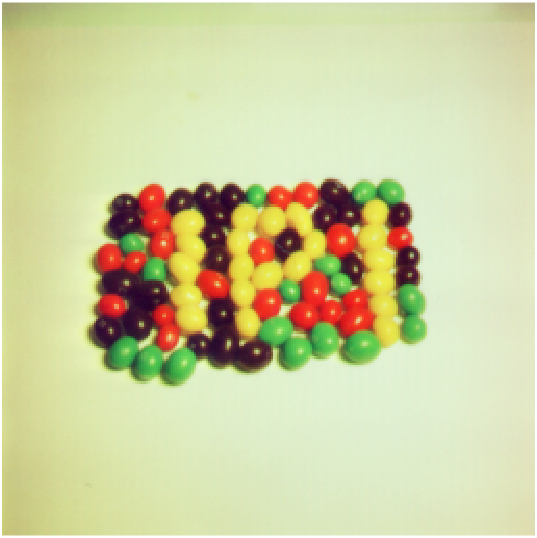}%
		}\,
	\subfloat[T=800(254,1)]{\includegraphics[width=0.7in]{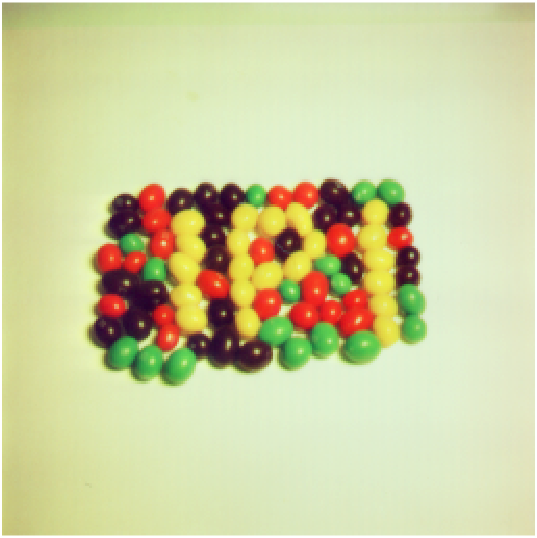}%
		}\,
	\subfloat[T=1500(275,1)]{\includegraphics[width=0.7in]{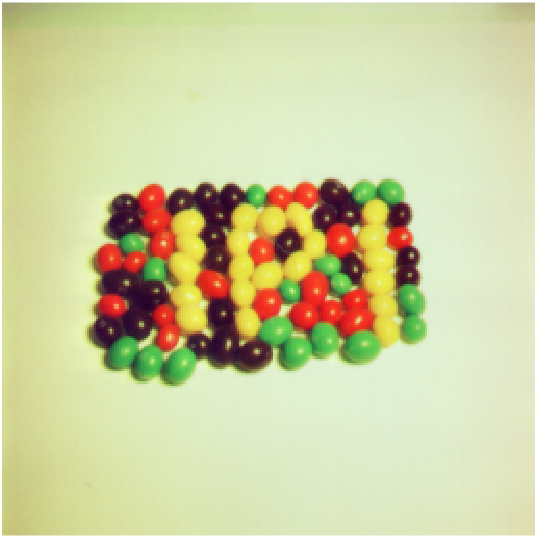}%
		}

    \subfloat[PQARAF:T=0]{\includegraphics[width=0.7in]{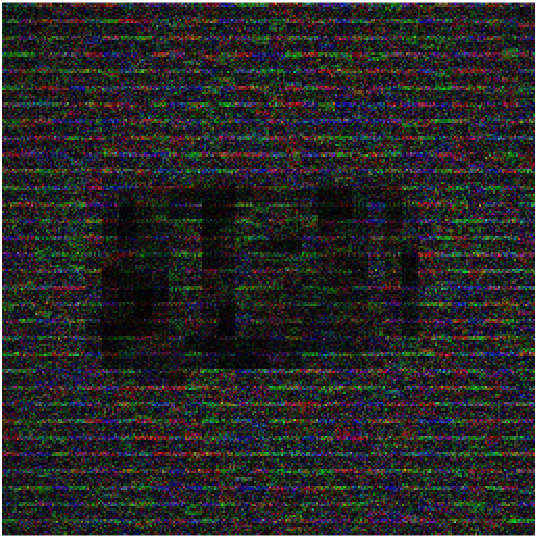}%
		}\,
	\subfloat[T=200(182,1)]{\includegraphics[width=0.7in]{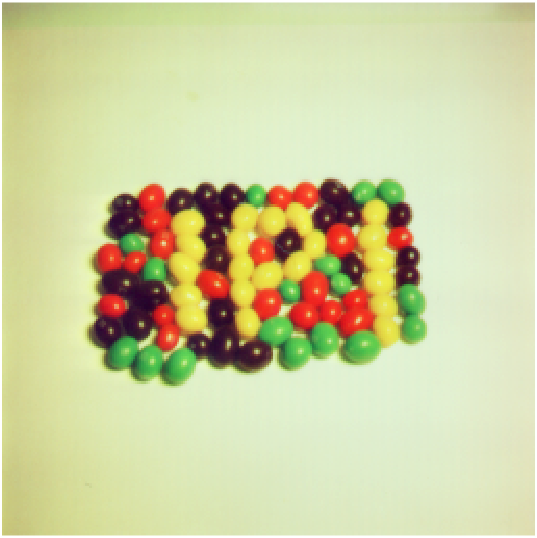}%
		}\,
	\subfloat[T=300(274,1)]{\includegraphics[width=0.7in]{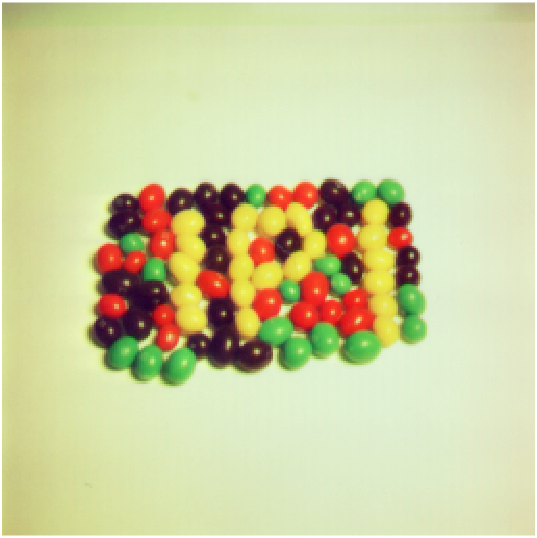}%
		}\,
	\subfloat[T=800(273,1)]{\includegraphics[width=0.7in]{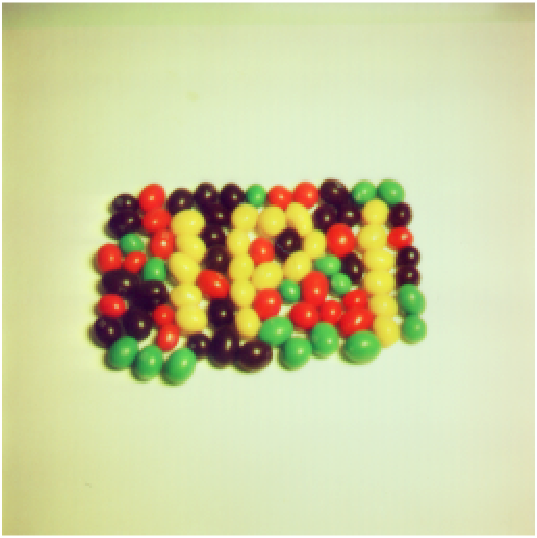}%
		}\,
	\subfloat[T=1500(274,1)]{\includegraphics[width=0.7in]{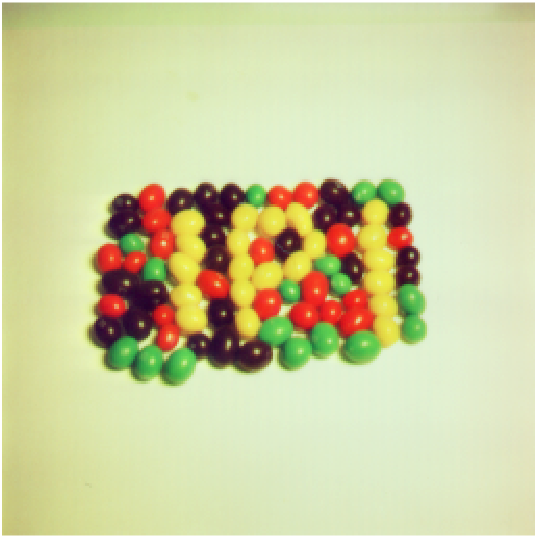}%
		}

    \subfloat[PQAdRAF:T=0]{\includegraphics[width=0.7in]{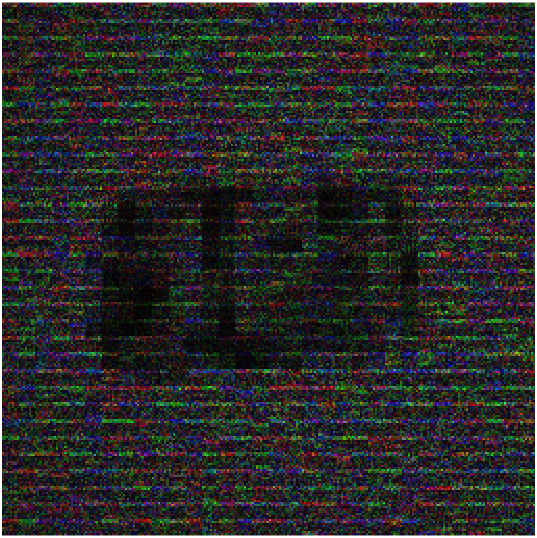}%
		}\,
	\subfloat[T=200(63,1)]{\includegraphics[width=0.7in]{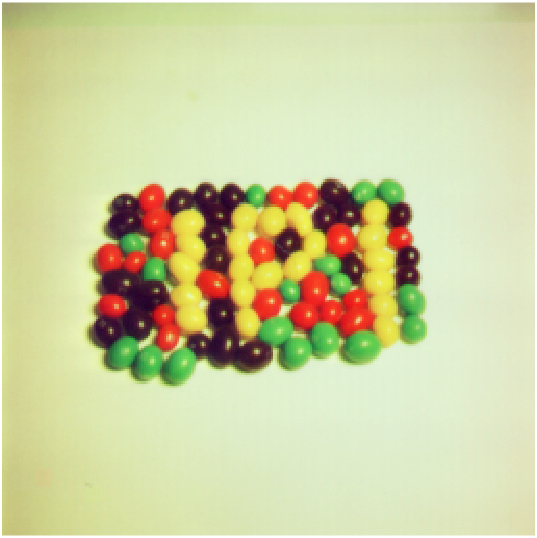}%
		}\,
	\subfloat[T=300(119,1)]{\includegraphics[width=0.7in]{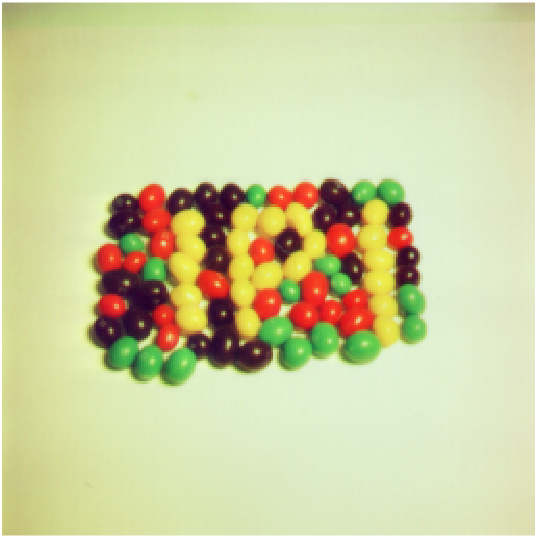}%
		}\,
	\subfloat[T=800(275,1)]{\includegraphics[width=0.7in]{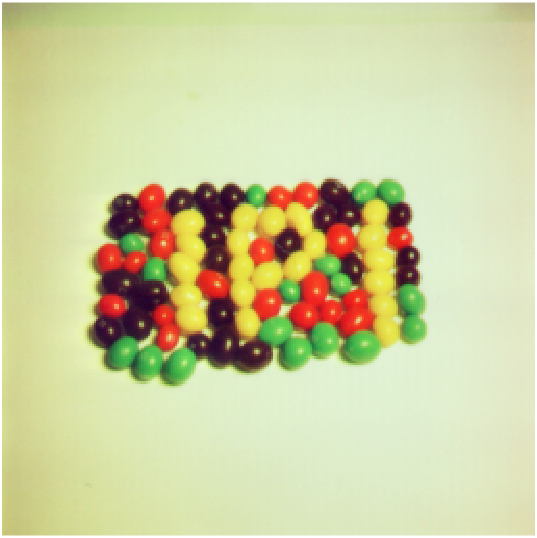}%
		}\,
	\subfloat[T=1500(275,1)]{\includegraphics[width=0.7in]{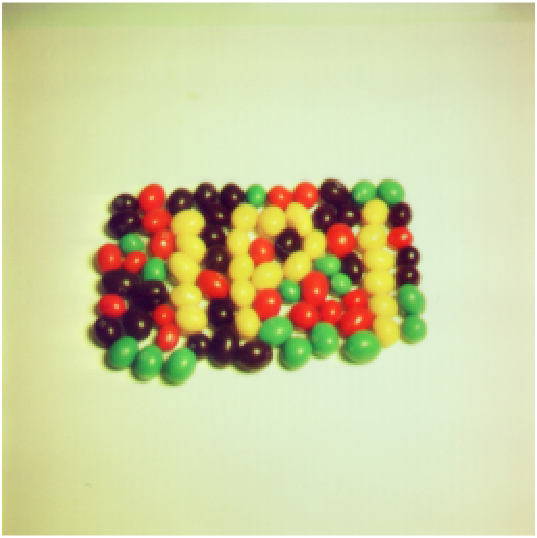}%
		}
	\caption{Key iteration steps of color image}
    \label{pic_gen_all}
\end{figure}

\noindent\textit{Note:} Slight visual artifacts (e.g., overall darkening in PQRAF at $T = 200$) are caused by MATLAB’s \texttt{imshow} function, which adjusts dynamic range based on non-converged pixel blocks. Once convergence is reached across the full image, rendering returns to normal.

\subsection{Comprehensive Comparison on Color Images}
In the previous sections, we presented experimental comparisons of (1) Wirtinger Flow vs. Amplitude Flow methods in Figure~\ref{QWFsvsQAFs}, (2) different variants within the QRAF family in Figure~\ref{PQRAFs}, and (3) real vs. quaternion-based algorithms in Figure~\ref{QRAFvsRAF}. To provide a more comprehensive evaluation similar to \cite{GUO2025110986}, we now include an overall comparison of all quaternion-based algorithms on a diverse set of standard color images.
\begin{figure*}[htbp]
	\centering
	\subfloat[]{\includegraphics[width=1.2in]{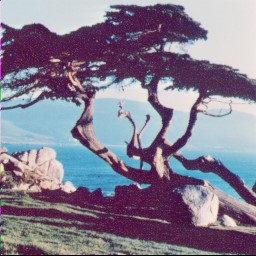}}\,%
	\subfloat[]{\includegraphics[width=1.2in]{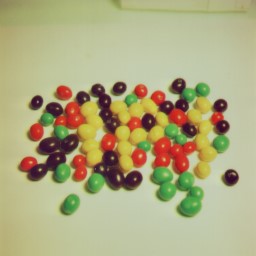}}\,%
	\subfloat[]{\includegraphics[width=1.2in]{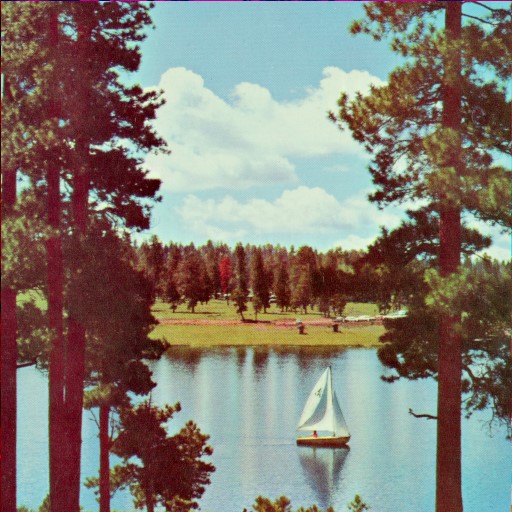}}\,%
	\subfloat[]{\includegraphics[width=1.2in]{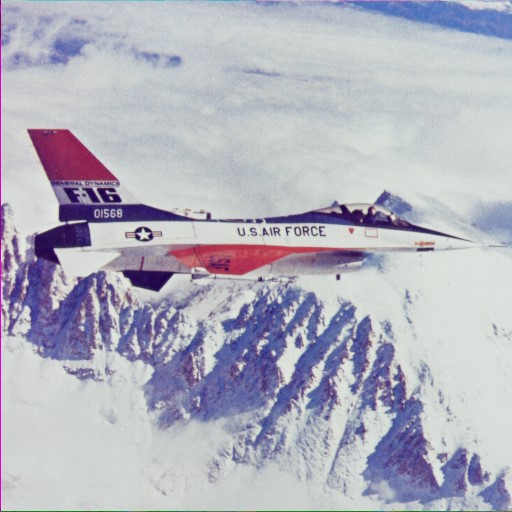}}\,%

    \subfloat[]{\includegraphics[width=1.2in]{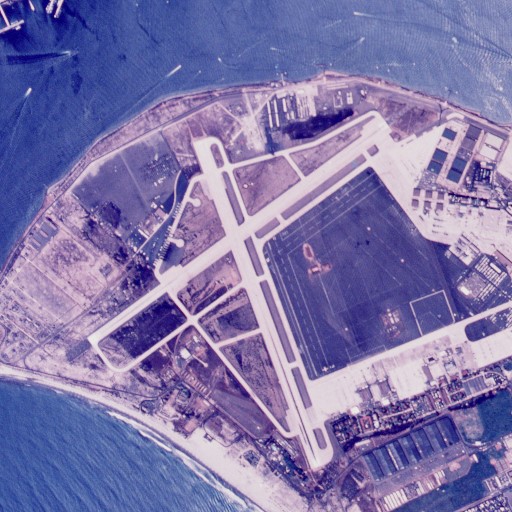}}\,%
	\subfloat[]{\includegraphics[width=1.2in]{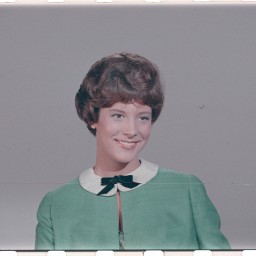}}\,%
	\subfloat[]{\includegraphics[width=1.2in]{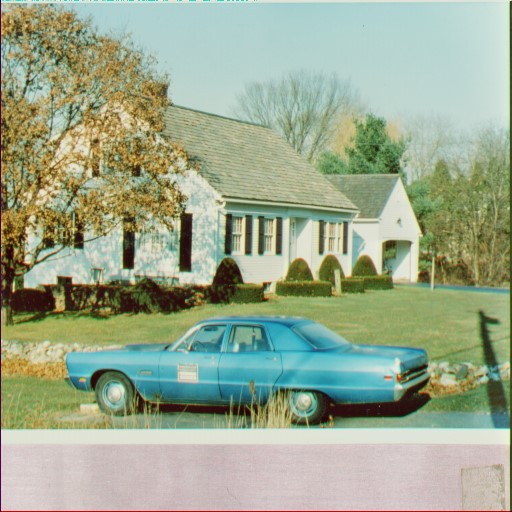}}\,%
	\subfloat[]{\includegraphics[width=1.2in]{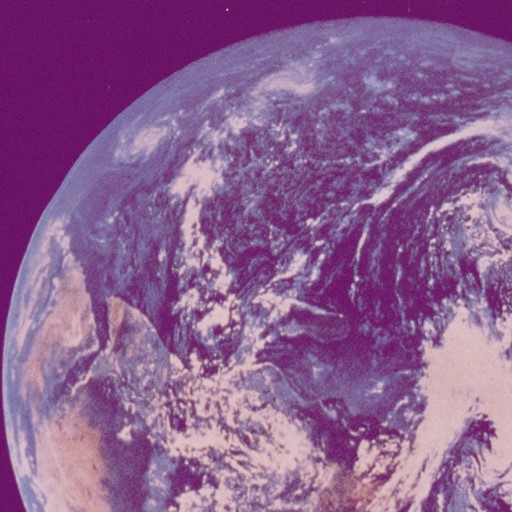}}\,%
	\caption{Color image set}
	\label{Q_algs_all}
\end{figure*}

The test image set, shown in Figure~\ref{Q_algs_all}, is selected from the USC-SIPI Image Database and includes a variety of textures and structures. We apply each algorithm to all test images and report the reconstruction results using two widely accepted performance metrics: Peak Signal-to-Noise Ratio (PSNR) and Structural Similarity Index Measure (SSIM). The quantitative results are summarized in Table~\ref{Table_full_comp}.

\begin{table*}[htbp]
	\caption{Comparison of algorithms in color image set by (PSNR, SSIM) }
	\centering
	\resizebox{\textwidth}{!}{\begin{tabular}{|l|c|c|c|c|c|c|c|c|}
		\hline
		Algorithm & $a$					& $b$ 		& $c$ 		&$d$ 		& $e$ 	& $f$	& $g$ & $h$ \\
		\hline
		PQWF &(15,0.68) 	&(11,0.34)		&(17,0.72)	&(11,0.30)		&(17,0.74) &(13,0.19)	&(14,0.52) &(18,0.77)\\
		
		PQRWF& (25,0.98)					&(29,0.98)			&(34,0.99)			&(28,0.97)	&(32,0.99) 	& (28,0.96) 	& (30,0.98)&(32,0.99) \\
		
		PQTAF &(23,0.92)					&(20,0.88)			&(24,0.93) 			&(20,0.81) 	&(24,0.94) 	&(22,0.79)&(22,0.89)&(23,0.93) \\
        
		PQPAF &(71,1) 	&(35,1) 	& (61,1)	&(38,1) 	&(39,1) &(38,1) &(45,1)&(45,1)\\
		
		PQRAF& (90,1)				&(49,1)			&(67,1)			&(42,1)	&(50,1)		&(67,1) &(48,1)& (46,1)\\
		
		PQIRAF &(74,1) 						&(57,1)			&(74,1)			&(59,1)			&(89,1)		&(74,1)	&(64,1) &(47,1)\\

         PQARAF &\textbf{(274,1)} 	&\textbf{(274,1)} 	& \textbf{(209,1)}	&\textbf{(255,1)} 	&\textbf{(269,1)} &\textbf{(202,1)} & \textbf{(277,1)}&\textbf{(292,1)}\\

        PQAdRAF &(137,1) 	&(129,1) 	& (78,1)	&(66,1) 	&(101,1) &(124,1) &(78,1)&(64,1)\\
		\hline
	\end{tabular}}
	
	\label{Table_full_comp}
\end{table*}

As shown in Table~\ref{Table_full_comp}, QRAF and its variants demonstrate superior overall performance compared to other quaternion-based algorithms. In particular, PQARAF and PQAdRAF consistently achieve the highest scores in both PSNR and SSIM across most test images. These results are well aligned with the convergence behavior observed in earlier experiments, further reinforcing the effectiveness and robustness of our proposed methods.

\medskip

In conclusion, this extended experiment provides both quantitative and visual evidence of the strong convergence properties of our QRAF family. The results demonstrate not only their effectiveness in reconstructing high-quality color images but also their efficiency in achieving rapid convergence. These findings reinforce the practical applicability and robustness of our proposed methods.
\end{document}